\documentclass[letterpaper]{article} %
\usepackage{aaai25}  %
\usepackage{times}  %
\usepackage{helvet}  %
\usepackage{courier}  %
\usepackage[hyphens]{url}  %
\usepackage{graphicx} %
\usepackage{natbib}  %
\usepackage{caption} %
\usepackage{algorithm}
\usepackage{algorithmic}

\usepackage{newfloat}
\usepackage{listings}
\DeclareCaptionStyle{ruled}{labelfont=normalfont,labelsep=colon,strut=off} %
\floatstyle{ruled}
\newfloat{listing}{tb}{lst}{}
\floatname{listing}{Listing}
\title{Toward Falsifying Causal Graphs Using a Permutation-Based Test}
\author{
    Elias Eulig\textsuperscript{\rm 1, 2,}\thanks{Work done at Amazon Research Tübingen.}\textsuperscript{\rm ,}\thanks{Correspondence to: \texttt{elias.eulig@dkfz.de}.}\quad
    Atalanti A. Mastakouri\textsuperscript{\rm 3}\quad 
    Patrick Blöbaum\textsuperscript{\rm 3} \\
    Michaela Hardt\textsuperscript{\rm 4,\textasteriskcentered}\quad
    Dominik Janzing\textsuperscript{\rm 3}
}
\affiliations{
    \textsuperscript{\rm 1}German Cancer Research Center (DKFZ)\quad
    \textsuperscript{\rm 2}Heidelberg University\\
    \textsuperscript{\rm 3}Amazon Research Tübingen\quad
    \textsuperscript{\rm 4}University Hospital Tübingen\\
}

\usepackage{hyperref}
\usepackage{booktabs}
\usepackage{url}
\usepackage{amsmath}
\usepackage{amssymb}
\usepackage{mathtools}
\usepackage{amsthm}

\usepackage{centernot}
\usepackage{enumitem}

\usepackage{mathtools}
\usepackage{amsthm}
\usepackage{bm}
\usepackage[capitalize,noabbrev]{cleveref}
\usepackage{tikz}
\usepackage{subcaption}

\let\oldFootnote\footnote
\newcommand\nextToken\relax

\renewcommand\footnote[1]{%
    \oldFootnote{#1}\futurelet\nextToken\isFootnote}

\newcommand\isFootnote{%
    \ifx\footnote\nextToken\textsuperscript{,}\fi}

\theoremstyle{plain}
\newtheorem{theorem}{Theorem}
\newtheorem{proposition}{Proposition}

\theoremstyle{definition}
\newtheorem{definition}{Definition}

\newtheorem{example}{Example}
\newtheorem{task}{Task}
\AfterEndEnvironment{theorem}{\noindent\ignorespaces}
\AfterEndEnvironment{proposition}{\noindent\ignorespaces}
\AfterEndEnvironment{lemma}{\noindent\ignorespaces}
\AfterEndEnvironment{definition}{\noindent\ignorespaces}
\AfterEndEnvironment{example}{\noindent\ignorespaces}
\AfterEndEnvironment{task}{\noindent\ignorespaces}

\usepackage{siunitx}
\usetikzlibrary{shapes,decorations,arrows,calc,arrows.meta,fit,positioning}
\tikzset{
    -Latex,auto,node distance =1 cm and 1 cm,semithick,
                                          state/.style ={font=\small, ellipse, draw, minimum width = 0.1 cm},
point/.style = {circle, draw, inner sep=0.04cm,fill,node contents={}},
bidirected/.style={Latex-Latex,dashed},
                  el/.style = {inner sep=2pt, align=left, sloped}
}
\newcommand{\inlinearc}[2]{\kern-5pt\begin{tikzpicture}[baseline=-0.5ex]%
\node (A) {};%
\node (B) [right of=A, node distance=#1] {};%
\path (A) edge[#2] (B);%
\end{tikzpicture}\kern-5pt}

\newcommand{\graph}{{\mathcal{G}}} %
\newcommand{\Ggiven}{{\hat{\mathcal{G}}}} %
\newcommand{\Gtrue}{{\mathcal{G}^{*}}} %
\newcommand{\verts}{\bm{V}} %
\newcommand{\edges}{\mathcal{E}} %
\newcommand{\Egiven}{\hat{\mathcal{E}}} %
\newcommand{\Etrue}{\mathcal{E}^{*}} %
\newcommand{\vars}{\bm{X}} %
\newcommand{\an}[2]{\text{Anc}_{#1}^{#2}} %
\newcommand{\nd}[2]{\text{ND}_{#1}^{#2}} %
\newcommand{\Pa}[1]{\text{Pa}_{{#1}}} %
\newcommand{\pa}[2]{\text{Pa}_{#1}^{#2}} %
\newcommand{\tpa}[1]{\text{T}^{#1}_\text{Pa}} %

\newcommand{\CI}{\mathrel{\text{\scalebox{1.07}{$\perp\mkern-10mu\perp$}}}}
\newcommand{\nCI}{\centernot{\CI}\mkern-4mu}

\newcommand{\joint}{P} %
\newcommand{\data}{\mathcal{D}} %
\newcommand{\thgiven}{\hat{\bm{\theta}}} %
\newcommand{\thtrue}{\bm{\theta}^{*}} %

\makeatletter
\newcommand{\vmd}[1][\@nil]{%
  \def\tmp{#1}%
   \ifx\tmp\@nnil
        V_{\text{MD}}
    \else
        V_{\text{MD}}^{{#1}}
    \fi}
\makeatother
\makeatletter
\newcommand{\fmd}[1][\@nil]{%
  \def\tmp{#1}%
   \ifx\tmp\@nnil
        \phi_{\text{MD}}
    \else
        \phi_{\text{MD}}^{{#1}}
    \fi}
\makeatother
\makeatletter
\newcommand{\vlmc}[1][\@nil]{%
  \def\tmp{#1}%
   \ifx\tmp\@nnil
        V_{\text{LMC}}
    \else
        V_{\text{LMC}}^{{#1}}
    \fi}
\makeatother
\makeatletter
\newcommand{\flmc}[1][\@nil]{%
  \def\tmp{#1}%
   \ifx\tmp\@nnil
        \phi_{\text{LMC}}
    \else
        \phi_{\text{LMC}}^{{#1}}
    \fi}
\makeatother
\makeatletter
\newcommand{\vtpa}[1][\@nil]{%
  \def\tmp{#1}%
   \ifx\tmp\@nnil
        V_{\text{TPa}}
    \else
        V_{\text{TPa}}^{{#1}}
    \fi}
\makeatother

\usepackage{xspace}
\makeatletter
\DeclareRobustCommand\onedot{\futurelet\@let@token\@onedot}
\def\@onedot{\ifx\@let@token.\else.\null\fi\xspace}

\def\eg{e.g\onedot} \def\Eg{E.g\onedot}
\def\ie{i.e\onedot} 
\def\cf{c.f\onedot} 
 
\def\wrt{w.r.t\onedot}
\def\iid{i.i.d\onedot}
\makeatother

\newcommand{\figref}[1]{Fig.~\ref{#1}}
\newcommand{\tabref}[1]{Tab.~\ref{#1}}
\newcommand{\secref}[1]{Sec.~\ref{#1}}
\newcommand{\suppref}[1]{Supp.~\ref{#1}}

\newcommand{\thref}[1]{Theorem~\ref{#1}}
\newcommand{\defref}[1]{Definition~\ref{#1}}
\newcommand{\propref}[1]{Proposition~\ref{#1}}

\newcommand{\taskref}[1]{Task~\ref{#1}}

\newcommand{\continuation}{??}

\newcommand{\figrefs}[1]{Fig.\ \labelcref{#1}}
\newcommand{\tabrefs}[1]{Tab.\ \labelcref{#1}}

\newcommand{\defrefs}[1]{Definition \labelcref{#1}}

\newtheorem{Assumption}{Assumption}

\begin{document}

\maketitle
\begin{abstract}
\looseness=-1Understanding causal relationships among the variables of a system is paramount to explain and control its behavior. For many real-world systems, however, the true causal graph is not readily available and one must resort to predictions made by algorithms or domain experts. Therefore, metrics that quantitatively assess the goodness of a causal graph provide helpful checks before using it in downstream tasks. Existing metrics provide an \textit{absolute} number of inconsistencies between the graph and the observed data, and without a baseline, practitioners are left to answer the hard question of how many such inconsistencies are acceptable or expected. Here, we propose a novel consistency metric by constructing a baseline through node permutations. By comparing the number of inconsistencies with those on the baseline, we derive an interpretable metric that captures whether the graph is significantly better than random. Evaluating on both simulated and real data sets from various domains, including biology and cloud monitoring, we demonstrate that the true graph is not falsified by our metric, whereas the wrong graphs given by a hypothetical user are likely to be falsified.
\end{abstract}

\section{Introduction}\label{sec:introduction}
Directed Acyclic Graphs (DAGs) are a core concept of causal reasoning as they form the basis of structural causal models (SCMs) which can be used to predict the effect of interventions in a causal system or even to answer counterfactual questions. For that reason, they have numerous applications, including biology~\citep{sachs2005,huber2007,pingault2018}, medicine~\citep{shrier2008} and computer vision~\citep{wang2013,chalupka2015,wang2020}.

\begin{figure}[tb]
    \centering
    \includegraphics[width=1.0\linewidth]{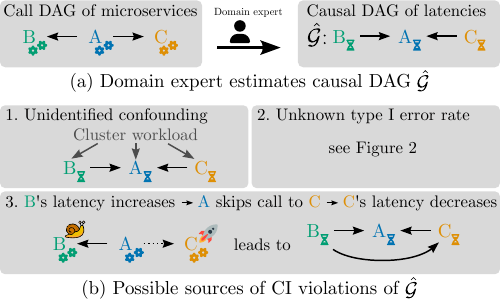}
    \vspace{-\baselineskip}
    \caption{\label{fig:examplefig} (a) For monitoring micro-service architectures, domain experts may invert the call graph to obtain a causal graph of latencies and error rates \citep{Budhathoki2022}; (b) However, there are multiple reasons why this graph might violate independence statements on observed data.}
\end{figure}

Nevertheless, for many real-world systems the true causal relationships, represented in the form of a DAG, are often not readily available. If randomized controlled trials are not possible, inferring the DAG from passive observational data alone is a hard problem that rests on strong assumptions on statistical properties (e.g. causal faithfulness), functional relationships, noise distributions, or graph constraints \citep{shimizu2006,peters2014,chickering1996,chickering2004,claassen2013,mooij2016}. These assumptions are often violated in practice. Instead of relying on discovery from observational data, domain experts can describe known dependencies in a system (\figref{fig:examplefig}; (a)). This approach, however, is subject to human error and incomplete knowledge. This is particularly problematic for applications in causal inference, where a wrong graph structure will lead to wrong conclusions about the effect of interventions.

Efforts towards quantitatively evaluating the consistency of a given graph (either originating from a domain expert or a causal discovery algorithm) using observational data alone are of utmost importance. However, existing attempts towards this direction report the raw number of inconsistencies between DAG and observed data without a baseline of how many such inconsistencies are to be expected in the first place \citep{textor2016,reynolds2022}. Striving for zero is unrealistic and the acceptable number of violations depends on many factors including the size and complexity of the graph. On real-world data, the fraction of conditional independence violations for expert-elicited graphs can be surprisingly high due to unidentified confounding, high type I error rates of the conditional independence tests \citep{shah2020}, or complex interactions between the variables (\figref{fig:examplefig}; (b)). Other methods require the existence and knowledge of DAGs from a related system \citep{pitchforth2013}, or are similarly difficult to interpret due to a missing baseline \citep{madigan1995,friedman2003}.

To overcome these drawbacks, in this work we develop a novel metric\footnote{An implementation is available in the Python package \textit{DoWhy} \cite{bloebaum2024}: \url{https://github.com/py-why/dowhy}.} to evaluate a given graph using observational data alone. Aiming to understand if a number of violations of a user-specified graph is high or low, we compare it against a baseline that we construct by randomly permuting its nodes. Through this comparison we can shed light onto the question if the violations of a user-specified graph are false-positives or point to real deficiencies of the graph.

\section{Related work}\label{sec:related_work}
\begin{figure}[tb]
    \centering
    \includegraphics[scale=0.75]{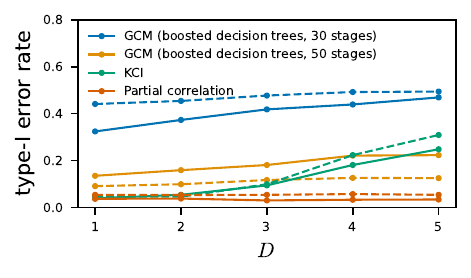}
    \vspace{-0.5\baselineskip}
    \caption{\label{fig:type1_ci} Type I error rate at $\alpha=5\%$ for different sizes $D$ of the conditioning set for one parametric (partial correlation) and two nonparametric CI tests, KCI~\citep{zhang2011} and GCM~\citep{shah2020}. Data (solid: $N$ = 100, dashed: $N$ = 500) were sampled from gaussian-linear conditionals. More details are given in \suppref{sec:app_type_1_ci}.}
\end{figure}

In this section, we review existing works that aim to quantify the consistency of a given DAG with observed data.

The R package \textit{dagitty} \citep{textor2016} implements functions to evaluate DAG-dataset consistency by testing conditional independence (CI) relations implied by the graph structure via the d-separation \citep{pearl1988} and global Markov condition. However, directly using the number of violations of graph-implied CIs as a metric to evaluate a given graph is not suitable for real-world applications because there exists no nonparametric CI test with valid level\footnote{Valid test would mean that for some CI test and sample size $N$, any distribution from the null is rejected at most with probability of the prespecified significance level $\alpha$ \citep{shah2020}.} over all distributions and thus the probability of type I errors remains unknown in practice \citep[Th. 2]{shah2020}, \cf \figref{fig:type1_ci}. Therefore, without a baseline comparison, the raw number of violations (absolute or fraction) does not provide the user with a meaningful measure of whether or not the observed inconsistencies of the given graph are significant.

\citet{reynolds2022} validate a given DAG (6 nodes, 8 edges) that relates exposure to spaceflight environment to the performance and health outcomes of rats and mice. Together with testing CIs implied by d-separations using \textit{dagitty}, they test whether dependencies implied by the graph lead to marginal dependencies in the observations (faithfulness). Similar to \cite{textor2016} interpreting the results of those tests is challenging, without a baseline comparison. Our metric overcomes this drawback by providing such a baseline comparison to estimate whether the observed inconsistencies are significant.

\citet{pitchforth2013} suggest a number of questions to validate expert-elicited Bayesian Networks (BNs). In particular, the authors propose to validate a given BN by verifying that it is similar to BNs from the same domain already established in the literature. Nevertheless, the framework does not provide a quantified measure and many of the questions assume the existence of comparison models in the literature, which may not be available in many domains. In contrast, we propose a quantitative metric, constructed via a surrogate baseline and not reliant on the existence of similar models from the same domain.

Finally, numerous works exist on Bayesian structure learning \citep[e.g.,][]{madigan1995,friedman2003,giudici2003} where the posterior $p(\graph \mid \mathcal{D})\propto p(\mathcal{D}\mid \graph)~p(\graph)$ of a graph $\graph$ is estimated given some data $\mathcal{D}$. While the likelihood $p(\mathcal{D} \mid \graph)$ could be used as a measure for the goodness of $\graph$, the problem of a missing baseline remains.

\section{Background}\label{sec:background}
\subsection{Notation}\label{sec:notation}
This work aims to evaluate the consistency of a given DAG $\Ggiven = (\verts, \Egiven)$ with vertices $\verts=\{1, ..., n\}$, random variables $\vars = \{X_i: i \in \verts\}$, and edges $\Egiven \subseteq \verts^2$ using observations $\data=\{\vars^{(1)}, ..., \vars^{(N)}\}$ \iid sampled from the joint distribution\footnote{In the following, we will assume that this $\joint$ has a density (\wrt a product measure).} $\joint(\vars)$. We assume the existence of some unknown true causal DAG $\Gtrue$ and that $\joint(\vars)$ satisfies the causal Markov condition relative to $\Gtrue$.

In the following, we refer to $X \CI_\data Y | Z$ as the outcome of a CI test using $\data$, where $X$ and $Y$ are conditionally independent given $Z$ (and likewise $X \nCI_\data Y | Z$ to denote conditional dependency). Note that whether independence is rejected is both a property of $\data$ and the choice of a particular CI test. In this work, to denote that a set of nodes $Z$ d-separates~\citep{pearl1988} $X$ from $Y$, we write: $X \CI_\graph Y | Z$.

For some graph $\graph=(\verts, \edges)$, we call a node $i$ a parent of node $j$ if $(i,j)\in\edges$ and denote with $\pa{j}{\graph}$ the set of all parents of $j\in\graph$ ($\pa{j}{\graph}=\emptyset$ iff $j$ is a root node). Furthermore, for nodes $i,j$ if there exists a direct path $i\to\dots\to j$ we call $i$ an ancestor of node $j$ and $j$ a descendant of $i$. We denote with $\an{j}{\graph}$ the set of all ancestors of $j\in\graph$ and with $\nd{i}{\graph}$ the set of all non-descendants of $i\in\graph$. 

\subsection{Validating local Markov conditions}\label{sec:validation_lmc}
One common approach to evaluate a DAG using observed data is by means of statistical testing of independence relations implied by $\Ggiven$ \citep[\eg][]{textor2016, reynolds2022}. In the following, we will formally introduce this test and elaborate on why it is unsuitable to use as a metric directly. In \secref{sec:node_permutation} we will then present a baseline to overcome said drawbacks.

One of the standard assumptions in causal inference is the causal Markov condition, which allows us to factorize a joint probability distribution $\joint(\vars)$ over the variables $\vars$ of $\graph$ into the product

\begin{align}
\joint(\vars) = \joint(X_1, X_2, ..., X_n) = \prod_{i=1}^n \joint(X_i | \pa{i}{\graph})~.
\end{align}
We can equivalently formulate this as the parental Markov condition:
\begin{theorem}[Parental Markov condition~\citep{pearl2009}]\label{th:lmc}
A probability distribution $P$ is Markov relative to a DAG $\graph = (\verts, \edges)$, iff $X_i \CI_P \nd{i}{\graph} \setminus\pa{i}{\graph} \mid \pa{i}{\graph}$.
\end{theorem}
For the remainder of this work we will use the terms parental Markov condition and local Markov condition (LMC) interchangeably.
To test whether a distribution satisfies the parental Markov condition relative to a DAG we can thus list the CIs entailed by the DAG via the graphical d-separation criterion~\citep{pearl1988} and test whether those are satisfied or violated by the distribution at hand.
\begin{definition}[Parental triples]\label{def:parental_triple}
For some graph $\graph$, we refer to the ordered triple $(i, j\in\nd{i}{\graph}\setminus\pa{i}{\graph}, Z=\pa{i}{\graph})$ as parental triple and denote the set of all such triples implied by $\graph$ as $\tpa{\graph}$.
\end{definition}
According to \thref{th:lmc} every parental triple $(i,j,Z)\in \tpa{\graph}$ implies the CI $X_i \CI_P X_j \mid Z$. Similar to  \citet{textor2016} we will now introduce the notion of violations of LMC:
\begin{definition}[Violations of LMCs]\label{def:violations_lmc}
We denote with $\vlmc[\Ggiven,\data]$ the set of triples $(i, j, Z) \in \tpa{\Ggiven}$, for which $(i,j)$ is an ordered pair, and we observe LMC violations on data $\data$, \ie
\begin{align}
     \vlmc[\Ggiven,\data] = \left\{ (i, j, Z) \in \tpa{\Ggiven}:  X_i \nCI_\data X_j \mid Z \right\}~. \label{eq:V_lmc}
\end{align}
Furthermore, we denote the fraction of LMC violations with %
\begin{align}
     \flmc[\Ggiven,\data] = \frac{| \vlmc[\Ggiven,\data] |}{|\tpa{\Ggiven}|}~.\label{eq:rel_lmc}
\end{align}
\end{definition}
Using the two metrics from \defref{def:violations_lmc} directly to measure the goodness of a user-given graph is not suitable for real-world applications due to the reasons detailed in \secref{sec:related_work}, particularly the unknown type I error rate of a CI test (\cf~\citet{shah2020} \& \figref{fig:type1_ci}).

\section{A baseline for violations of local Markov conditions}\label{sec:node_permutation}
As discussed in the previous sections, the metrics $V_\text{LMC}^{\Ggiven, \data}$ and $\phi_\text{LMC}^{\Ggiven, \data}$ are insufficient to measure the consistency of a graph. In the following, we therefore derive a baseline to compare the number of LMC violations $V_\text{LMC}^{\Ggiven, \data}$ to. 

\subsection{Finding a suitable baseline}
Motivated by a very skeptical view whether the user specified graph is related to the observed independence structure at all, we are interested in a baseline that is a \textit{random} draw of a set of conditional independence statements. For example, consider the dependence structure of microsvervices in a distributed system, where multiple effects can lead to violations of independence statements on observed data (\figref{fig:examplefig}).

\looseness=-1In general, there can be different reasons why the pattern of observed conditional independences appears unrelated to $\Ggiven$. On the one hand, the domain expert who provided $\Ggiven$ may have messed up causal links and directions entirely. But even if all the links of $\Ggiven$ are correct, additional confounding and violations of faithfulness\footnote{Note, that while faithfulness violations cannot lead to CI violations for $\Ggiven$, it can make a random set of CI statements equally data consistent.} can mess up the independence structure. In both cases, DAG and independences appear random {\it relative to each other} regardless of whether we think the DAG or the pattern of independences to be random. Our experiments show that `better than random' is a surprisingly high bar, and both domain experts and causal discovery algorithms fail to meet it in many settings.

We now introduce a set of properties that draws from the random baseline should satisfy:
\begin{itemize}[leftmargin=18pt]
    \item [P1:]They should infer the same number $m$ of conditional independences as the given graph $\Ggiven$.
    \item [P2:]They should partition $m$ into $m_r$ conditional independences with conditioning sets of size $r$ with the same numbers $m_r$ as $\Ggiven$ does.
    \item [P3:]Conditional independences inferred by draws from the baseline should be closed under the semi-graphoid axioms~\citep{pearl1986,geiger1990}.
\end{itemize}
P1 ensures that the number of observed LMC violations on the baseline are comparable to those observed on $\Ggiven$. \Eg a random baseline inferring fewer CIs than $\Ggiven$ will also result in much fewer LMC violations.\footnote{This is analogous to a classifier (predicting CIs), that classifies only few samples as positive, hence exhibiting few false positives (LMC violations).} 
If $\Ggiven$ is sparse because it has been drawn with the implicit intention of explaining only {\it highly} significant dependences,  we should not reject it just because it shows many LMC violations with respect to {\it our} significance level. Instead, we should rather benchmark $\Ggiven$ against a random guess that infers equally many independences.  
P2 ensures that type I and type II errors, when testing implied CIs on the baseline, are comparable to those on $\Ggiven$. \Eg a random baseline inferring mostly CIs with small conditioning sets may exhibit fewer LMC violations than $\Ggiven$ by virtue of a smaller type I error rate (\cf~\figref{fig:type1_ci}). Lastly, P3 ensures that the CI statements from the baseline do not imply additional CIs (which are not already in the set) via the semi-graphoid axioms. For example, due to decomposition, $X\CI Y \cup W | Z \Rightarrow X \CI Y | Z ~ \& ~ X \CI W | Z$. 

\looseness=-1 A natural choice for a baseline that satisfies all the requirements above can be constructed by sampling node-permutations of $\Ggiven$. More formally, let $S_n$ denote the set of permutations $\pi$ on the vertices $\{1, ..., n\}$ of some graph $\graph$. For any permutation $\pi \in S_n$ we denote with $\sigma_\pi(\graph)$ the graph for which the edge $i \to j$ exists iff $\pi(i) \to \pi(j)$ exists in $\graph$. Because of the one-to-one correspondence between $\pi$ and $\sigma$ we will drop the subscript and in the following refer to $\sigma\in S_\graph$ as one node-permutation of $\graph$.

We can then construct our baseline by sampling random node-permutations $\sigma\in S_\Ggiven$ of the given graph $\Ggiven$. Let $O(\Ggiven)$ define the orbit of $\Ggiven$ under $S_\Ggiven$, i.e., the set of DAGs obtained via permutations. Since all DAGs in $O(\Ggiven)$ imply the same CIs as $\Ggiven$ up to renaming of variables, they satisfy P1 and P2. Furthermore, they satisfy P3 since the set of CIs entailed by a graph are closed under the semi-graphoid axioms.\footnote{
This is because faithful distributions exist, which satisfy only CIs corresponding to d-separation.}\footnote{This choice of baseline comes with another feature, namely the implicit testing of violations of faithfulness (\cf \suppref{sec:supp_faithfulness}).}

Note that the mapping $S_\Ggiven\to O(\Ggiven)$ defined by $\sigma \mapsto \sigma(\graph)$ is in general not one-to-one because there will often be a non-trivial subgroup that leaves $\Ggiven$ invariant (the stabilizer subgroup $\text{Stab}(\Ggiven)$ of $S_\Ggiven$).\footnote{\Eg the permutation $(2, 1, 3)$ leaves the graph $X_1\to X_3 \leftarrow X_2$ invariant.}

\begin{proposition}\label{prop:uniform_sampling}
Uniform sampling of permutations from the set of all node permutations $S_\Ggiven$ results in uniform sampling from the DAGs in the orbit $O(\Ggiven)$.
\end{proposition}

A proof is provided in \suppref{sec:proof_uniform_sampling}. Further, $O(\Ggiven)$ decomposes into Markov equivalence classes of equal size. This can easily be seen by the same argument when we consider the action of $S_\Ggiven$ on the set of Markov equivalence classes and introduce the corresponding (larger) stabilizer subgroup. Thus, uniform sampling of permutations from the set of all node permutations $S_\Ggiven$ also results in uniform sampling from the Markov equivalence classes in $O(\Ggiven)$.

\subsection{A permutation test to evaluate DAGs}
Using the above baseline we state the following null hypothesis:

\noindent{\bf Hypothesis $H_0$:} The DAG $\Ggiven$ is drawn uniformly at random from some distribution $Q$ on the set of DAGs that is invariant under permutations of nodes, that is $Q(\graph)=Q(\sigma (\graph))$ for all $\sigma \in S_\graph$.

To test this hypothesis we consider the number of violations $|\vlmc[\Ggiven, \data]|$ as test statistics and build the null via $|\vlmc[\sigma(\Ggiven), \data]|$:
\begin{align}
    p_\text{LMC}^{\Ggiven, \data} = \text{Pr}\left( |\vlmc[\sigma(\Ggiven), \data]| \leq |\vlmc[\Ggiven, \data]| \right) \label{eq:p_lmc} ~.
\end{align}
\begin{proposition}\label{prop:p_lmc_p-value}
$p_\text{LMC}$ is a valid p-value, \ie if $H_0$ is true, then
$\text{Pr}\left( p_\text{LMC}^{\Ggiven, \data}\leq \alpha\right) \leq \alpha~.$
\end{proposition}
A proof is provided in \suppref{sec:proofpvalue}. Computing the quantity in \eqref{eq:p_lmc} using all $n!$ permutations is infeasible for large $n$. Therefore, we approximate it via Monte Carlo sampling with $T$ random permutations: $\{\sigma_i \sim S_\Ggiven \}_{i=0}^T$. For an estimated p-value we can also report binomial proportion confidence intervals.

Another quantity that proves to be useful in practice is the fraction of DAGs in $O(\Ggiven)$ that are Markov equivalent to $\Ggiven$. To this end we first define 
\begin{align}
    \vtpa[\graph', \graph] = \{ (X_i, X_j, Z) \in \tpa{\graph'}: X_i \nCI_\graph~ X_j \mid Z \}  \label{eq:V_dsep}
\end{align}
as the set of all ordered triples that are parentally d-separated in $\graph'$ but not d-separated in $\graph$. 
\begin{proposition}\label{prop:markov_eq}
Suppose some given graph $\Ggiven$. If $\vtpa[\sigma(\Ggiven), \Ggiven]=\emptyset$, then $\sigma(\Ggiven)$ and $\Ggiven$ are Markov equivalent.
\end{proposition}
The proof is provided in \suppref{sec:proofmec}.
Using this graphical criterion we can define a second metric
\begin{align}
    p_\text{TPa}^{\Ggiven} :=& \text{Pr}\left( |\vtpa[\sigma(\Ggiven), \Ggiven]| \leq |\vtpa[\Ggiven, \Ggiven]| \right) \notag \\
    =& \text{Pr}\left( |\vtpa[\sigma(\Ggiven), \Ggiven]| = 0 \right) ~,
\end{align}
which can be used to measure how informative the CI structure of $\Ggiven$ is about the possible causal orderings. If $p_\text{TPa}^{\Ggiven} > \alpha$, for some prespecified threshold $\alpha$, then the number of Markov equivalent DAGs in $O(\Ggiven)$ is large and consequently $\Ggiven$ provides us with limited information about the true graph in the sense of testing of CIs. For an information-theoretic interpretation of our test see \suppref{sec:supp_information_view}.
\subsection{Interpretation of $p_\text{LMC}$ and $p_\text{TPa}$}\label{sec:interpretation_metrics}
\looseness=-1We note that p-values are often misinterpreted and misused in practice \cite{vidgen2016,wasserstein2016,amrhein2019} and therefore provide the following interpretation based on Popper's theory of falsification \cite{popper1934}. According to this, every scientific theory must admit potential falsifiers, \ie measurable observations that would falsify the theory. Increasing confidence in such theory can then only come from observations that it permits, \ie as it withstands attempts to falsify it.

The test $p_\text{TPa}^{\Ggiven}$ provides us with a measure of the falsifiability of a given graph $\Ggiven$. If the number of random DAGs that are Markov equivalent to $\Ggiven$ is large (and consequently $p_\text{TPa}^{\Ggiven}$ is large) this limits the falsifiability of $\Ggiven$ via CI testing. Contrary, if $p_\text{TPa}^{\Ggiven}$ is small, the CIs entailed by $\Ggiven$ are `characteristic' and it admits many potential falsifiers. Based on these considerations, we propose the following interpretation of our tests for practitioners:
\vspace*{-0.5\baselineskip}
\begin{itemize}[leftmargin=16pt]
    \item [(a)]If $p_\text{TPa}^{\Ggiven} \leq \alpha$, $\Ggiven$ is falsifiable by testing implied CIs.
    \item [(b)]If (a) and further $p_\text{LMC}^{\Ggiven, \data} > \alpha$, then $\Ggiven$ is falsified.
    \item [(c)]If (a) and further $p_\text{LMC}^{\Ggiven, \data} \leq \alpha$, then there is no CI-based evidence against $\Ggiven$ and we cannot falsify $\Ggiven$ using our test. Consequently, $\Ggiven$ is corroborated (but not verified).
\end{itemize}
\vspace*{-0.5\baselineskip}

\subsection{Relating $p_\text{LMC}$ to the identification of cause-effect pairs}
We now want to present an interpretation of $p_\text{LMC}$ which is more closely related to \textit{causal} questions. To this end, consider the following task.
\begin{task}[Identification of unconfounded cause-effect pairs]\label{task:causal-effects}
Identify all ordered pairs $(X_i,X_j)$ such that there is a directed path from $X_i$ to $X_j$ and no $X_k$ with $k\neq i,j$ that has a directed path to $X_i$ and a directed path to $X_j$ that does not go through $X_i$.
\end{task}
We can then ask whether $\Ggiven$ performed better at \taskref{task:causal-effects} than a random baseline. However, in analogy to P1 above, the baseline should satisfy the following property:
\begin{itemize}[leftmargin=18pt]
    \item [P4:] Draws from the random baseline should infer the same number of unconfounded cause-effect pairs as the given graph $\Ggiven$.
\end{itemize}
Likewise to P1 \& P2, all DAGs in $O(\Ggiven)$ satisfy P4 since they preserve $\Ggiven$ up to a renaming of variables. By defining the number of wrongly inferred causal effects as test statistics and building the null via the node-permutation baseline, we can compute a p-value $p_\text{CE}$.

We will now see that under the assumption of faithfulness and knowledge of a single non-effect node, our metric $p_\text{LMC}$ is loosely related to $p_\text{CE}$. 

\begin{theorem}\label{th:unconfounded_cause_effect}
Suppose an additional node $X_0$ for which we know (\eg through strong domain knowledge) that it is not an effect of any of the $1,\dots, n$ nodes, but satisfies $X_0\nCI X_l,~\forall l=1, \dots, n$. Further, let the joint distribution $P(X_0, X_1, \dots, X_n)$ be Markov and faithful relative to some DAG $\graph$ with these $n+1$ nodes. Then $(X_i,X_j)$ with $i,j=1,\dots,n$ is an unconfounded cause-effect pair if and only if the following two conditions hold:
\begin{align} 
X_i &\nCI_P X_j \label{eq:cause-effect_dependence}\\
X_0 &\CI_P X_j\,| X_i \label{eq:cause-effect_independence}~.
\end{align}
\end{theorem}
The proof is provided in \suppref{sec:supp_proof_causal_effects}. We conclude that in our idealized scenario, testing whether $\Ggiven$ is better than random at \taskref{task:causal-effects}, is similar to testing whether $\Ggiven$ performs better than random in identifying all pairs $(X_i,X_j)$ for which both conditions \eqref{eq:cause-effect_dependence} and \eqref{eq:cause-effect_independence} hold. Here, the baseline is generated by the permutation group $S_{n-1}$, that is, the stabilizer group of $X_0$ in $S_n$.

\section{Experiments}\label{sec:experiments}
In the following we evaluate our proposed metric on synthetic and real data for which the true DAG is known or a consensus graph is established in the literature. Additionally, we introduce a novel dataset from cloud monitoring where the reversed call graph provides an estimate of the true causal graph (\figref{fig:examplefig}; (a)), thus providing a useful test beyond synthetic and existing real-world datasets\footnote{Code for running the experiments and the cloud monitoring dataset are available at \url{https://github.com/eeulig/dag-falsification}.}.

\subsection{Experimental setup}\label{sec:experimental_setup}
We conduct experiments with two different sources of given graphs: Emulated domain experts with partial knowledge of either a subset of nodes or a subset of edges of the true DAG $\Gtrue = (\verts, \Etrue)$ and causal discovery algorithms, which are a popular choice to estimate a DAG in the absence of domain expertise. Note that all $\Ggiven$ in our experiments differ from $\Gtrue$ in a structured, \ie non-random, way.

\paragraph{Node Domain Expert (DE-$\verts$)} This model mimics the situation where a domain expert knows all causal edges for some subset $K\subseteq\verts$ of the nodes in the system and knows the overall sparsity of the DAG. For all the other nodes, however, the expert has no domain
knowledge and thus assigns edges randomly between pairs of nodes not both in $K$. We define different levels of DE-$\verts$ to emulate  the fraction of nodes for which there exists domain knowledge, \ie $|K| / |\verts|$, where $|K|/|\verts|=0$ corresponds to the situation where the domain expert has no knowledge and $|K|/|\verts|=1$ corresponds to $\Ggiven=\Gtrue$. More details are given in \suppref{sec:node_expert}.

\paragraph{Edge Domain Expert (DE-$\edges$)} This model mimics a domain expert with edge-specific knowledge about $\Gtrue$. To construct $\Ggiven$, we randomly remove and flip some of the true edges and add some new ones. By construction $\Ggiven$ of a DE-$\edges$ is related to the Structural Hamming Distance (SHD)~\citep{acid2003,tsamardinos2006} and thus the desired similarity of a given graph can be controlled by means of the SHD. We characterize different DE-$\edges$ by the $\text{SHD}(\Ggiven, \Gtrue)$ they entail (or $\text{SHD}(\Ggiven, \Gtrue) / |\Etrue|$ to compare systems with different sparsity), where $\text{SHD}(\Ggiven, \Gtrue) / |\Etrue| = 0$ corresponds to  $\Ggiven=\Gtrue$. More details are given in \suppref{sec:edge_expert}.

\paragraph{Causal Discovery Algorithms} The proposed test can also be applied to DAGs inferred by causal discovery algorithms. However, as many algorithms use (some of the) CIs either explicitly (constraint-based) or implicitly (score-based) for constructing the DAG, evidence in favor of an inferred DAG can only come from those CIs that were neither used by the algorithm, nor implied, via semi-graphoid axioms~\citep{pearl1986,geiger1990}, by those CIs used by the algorithm. Nonetheless, we evaluate our test on graphs inferred by LiNGAM~\citep{shimizu2006}, CAM~\citep{buhlmann2014}, and NOTEARS~\citep{zheng2018}. We chose those algorithms as they are not solely based on CIs (note, however, that e.g. in LinGAM independence of noise entails CI). More details are given in \suppref{sec:app_results_causal_discovery}.
\subsection{Synthetic data and graphs}\label{sec:experiments_synthetic}
For evaluating our method on synthetic data we sample random $\Gtrue$ under the Erd\H{o}s-R\'{e}nyi model~\citep{erdos1959} with $n \in \{10,20,30\}$ nodes and an expected degree $d\in \{1, 2, 3\}$, denoted as ER-$n$-$d$. To generate synthetic data from $\Gtrue$, conditionals are modeled as additive noise models $X_i = f_i(\Pa{i}^{\Gtrue}) + N_i$ with $N_i$ sampled from a normal distribution and $f_i$ either being random (nonlinear) MLPs, or a linear combination of the node's parents. The exogenous variables are sampled from a standard normal, uniform, or Gaussian mixture distribution. 

For all experiments on synthetic data we sample $T=\num{e3}$ node permutations and use datasets with $N=\num{e3}$ observations. To investigate the effect of $N$ and $T$ on $p_\text{LMC}$ we run ablation studies on nonlinear data with $N,T \in \{\num{e1}, \num{e2}, \num{e3}, \num{e4}\}$.
More information on implementation and parameter choices is provided in \suppref{sec:app_experiments_synthetic}.
\begin{figure*}[t]
    \centering
    \begin{subfigure}{0.49\textwidth}
        \centering
        \includegraphics[width=\textwidth]{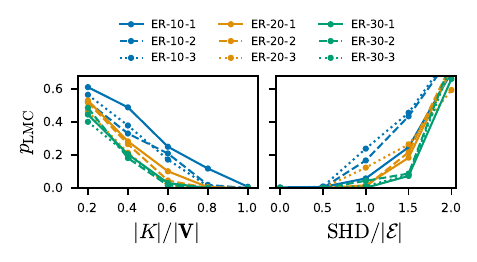}
        \vskip-0.5\baselineskip
        \caption{Synthetic (nonlinear mechanisms)}\label{fig:p_lmc_synthetic_nonlinear}
    \end{subfigure}
        \hfill
    \begin{subfigure}{0.49\textwidth}
        \centering
        \includegraphics[width=\textwidth]{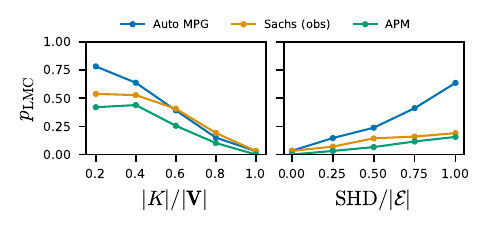}
        \vskip-0.5\baselineskip
        \caption{Real-world data}\label{fig:p_lmc_real}
    \end{subfigure}
    \caption{\label{fig:results_nonlinear_real}Mean $p_\text{LMC}$ for two types of domain experts, simulated via DE-$\verts$ (left; smaller numbers correspond to less domain knowledge) and DE-$\edges$ (right; smaller numbers correspond to more domain knowledge). On synthetic data (\subref{fig:p_lmc_synthetic_nonlinear}) for the true DAG ($|K|/|\verts| = 1$; $\text{SHD}/|\edges|=0$), we reject the null that the DAG is as bad as random with $\alpha=1\%$ for all configurations. $\Ggiven$ is falsified with the same $\alpha$ if $|K|/|\verts|\leq 0.6$ or $\text{SHD}/|\edges|\geq 1.5$. On real-world data (\subref{fig:p_lmc_real}), for the true DAG, we reject the null that it is as bad as random with $\alpha=5\%$ and $\Ggiven$ is falsified with the same $\alpha$ for $|K|/|\verts|\leq 0.8$ or $\text{SHD}/|\edges|\geq 0.5$ for all datasets.}
\end{figure*}
\subsection{Real data}\label{sec:experiments_real}
To evaluate our proposed metric on real-world data, we consider three datasets with established consensus graphs serving as ground truth. We provide further details on the data in \suppref{sec:app_experiments_real}.

\paragraph{Protein Signaling Network \citep{sachs2005}} This open dataset contains quantitative measurements of the expression levels of $n=11$ phosphorylated proteins and phospholipids in the human primary T cell signaling network. The $N=853$ observational measurements, corresponding to individual cells, were acquired via intracellular multicolor flow cytometry \cite{sachs2005}. The consensus DAG contains 19 edges ($d\approx3.45$).

\paragraph{Auto MPG~\citep{quinlan1993}} The Auto MPG dataset contains eight attributes (three multivalued discrete and five continuous) with the fuel consumption in miles per gallon (mpg) for $N=398$ unique car models. While the original use of the data was to predict mpg of a car, in line with previous works on causal inference~\citep[\eg][]{wang2017,teshima2021}, we use a consensus network ($n=6$, 9 edges, $d=3$).

\paragraph{Application Performance Monitoring (APM)} We collect trace data of microservices in a distributed system hosted on Amazon Web Services (AWS). The traces contain latency information on incoming and outbound requests of each service, averaged over \SI{20}{\min}, making observations approximately \iid. In total we ran the application for six days, leading to $N=432$ observations. We test the working hypothesis that the transpose of the dependency graph ($n=39$, 40 edges, $d\approx 2.05$) of the application is the true causal DAG (for a discussion of this hypothesis, see \secref{sec:node_permutation}).

\section{Results}\label{sec:results}
\subsection{Simulated graphs}
\noindent\textbf{Nonlinear Mechanisms} Figure \ref{fig:p_lmc_synthetic_nonlinear} depicts mean $p_\text{LMC}$ for 50 synthetic graphs of various size and sparsity with nonlinear mechanisms modeled using random MLPs. As expected, we find that the average $p_\text{LMC}$ monotonically decreases with increasing amount of domain knowledge for both models of domain experts. When the domain expert has complete knowledge ($\Ggiven = \Gtrue$, corresponding to $|K|/|\verts| = 1$ for DE-$\verts$ and $\text{SHD}/|\edges|=0$ for DE-$\edges$), we reject the hypothesis that the DAG is as bad as a random node permutation with significance level $\alpha=1\%$ for all configurations ($p_\text{LMC} < 0.005$).

\paragraph{Gaussian-Linear Mechanisms} In the supplemental, \figref{fig:p_lmc_synthetic_linear} we report $p_\text{LMC}$ for synthetic graphs of various size and sparsity and with linear-gaussian mechanisms. Similar to DAGs with nonlinear mechanisms, we observe that our metric strictly decreases with increasing amount of domain knowledge for both models of domain experts. If $\Ggiven = \Gtrue$, we would not falsify the true graph using our metric and significance level $\alpha=1\%$.

\begin{table}[tb]
    \centering
    \caption{\label{tab:real_results} Mean $p_\text{LMC}$ and standard deviation (over 50 trials) for the consensus graphs of the real-world data sets with 95\% confidence interval. For $\alpha=5\%$ we reject the hypotheses that the graphs are as bad as random ones, despite high fractions of violations  $\flmc[\Gtrue, \data]$ \eqref{eq:rel_lmc}.}
    \fontsize{9pt}{11pt}\selectfont
    \setlength{\tabcolsep}{4pt}
    \begin{tabular}{@{}lccc@{}} 
    \toprule 
& \textbf{Sachs} & \textbf{Auto MPG} & \textbf{APM} \\ \midrule
$p_\text{LMC}^{\Gtrue,\data}$ & $0.031 \pm 0.006$ & $0.03 \pm 0.0052$ & $0.00$\\
$95\%$  conf. int. & $[0.020, 0.042]$ & $[0.019, 0.04]$& --\\ \midrule
$\flmc[\Gtrue,\data]$ & $0.08$ & $0.50$ & $0.22$\\\bottomrule
 \end{tabular}

\end{table}
\paragraph{Effect of number of sampled permutations and number of observations}
Further, we investigate the effect of the number of permutations $T$ and sample size $N$ on $p_\text{LMC}$ for synthetic DAGs with nonlinear mechanisms (\tabref{tab:num_perms_samples}). To limit the running time of the experiment, we only evaluate $p_\text{LMC}$ for the true graph, \ie $\Ggiven=\Gtrue$. Here, we notice that on average our metric is consistently below 0.05, and therefore we would not reject the true graph with significance level $\alpha=5\%$. The only exception to this are graphs with few nodes (ER-10-$d$), evaluated on very few ($N=10$) samples.

\subsection{Real world applications}\label{sec:exp_real}
Figure \ref{fig:p_lmc_real} shows mean $p_\text{LMC}$ over 50 sampled given DAGs for the three real-world datasets. Similar to the experiments with synthetic data, we find that with increasing amount of domain knowledge (higher $|K|/|\verts|$, lower $\text{SHD}/|\edges|$) $p_\text{LMC}$ is strictly decreasing. When $\Ggiven=\Gtrue$ we reject the hypotheses that the given graphs are as bad as a random node permutation at $\alpha=5\%$ for all datasets.

Furthermore, we find that for all real-world datasets the fraction of LMC violations $\flmc[\Gtrue, \data]$ \eqref{eq:rel_lmc} is higher than the expected type I error rate of $5\%$ for the significance level $\alpha=5\%$ we used for all conditional independence tests throughout this work (\cf \tabref{tab:real_results}). Thus, using $\flmc[\Gtrue, \data]$ as a metric, we would falsely reject the true causal graph, na\"ively assuming our CI tests would have valid level.

\begin{table}[tb]
    \centering
    \caption{\label{tab:discovery_sachs} $p_\text{LMC}$ and SHD for graphs inferred by causal discovery on the \citet{sachs2005} data.}
    \fontsize{8pt}{10pt}\selectfont
    \setlength{\tabcolsep}{2.5pt}
    \begin{tabular}{@{}l
S[table-format=1.3(1)]
S[table-format=1.4(1)]
S[table-format=1.4(1)]@{}} 
\toprule                                & {\textbf{NOTEARS}}    & {\textbf{CAM}}   & {\textbf{LiNGAM}}\\\midrule
$p_\text{LMC}^{\Ggiven,\data}$          & 0.548 \pm 0.237       & 0.0724 \pm 0.0866  & 0.0362 \pm 0.129 \\
$\text{SHD}/|\edges|$                   & 2.78 \pm 0.241        & 1.88 \pm 0.198   & 1.06 \pm 0.0806  \\\bottomrule
\end{tabular}

\end{table}
\subsection{Causal discovery algorithms}\label{sec:exp_causal_discovery}
While the main scope of this work is to evaluate user-given graphs, we conduct additional experiments with $\Ggiven$ inferred via causal discovery. On the \citet{sachs2005} data we find that graphs inferred by NOTEARS and CAM are not significantly better than random, whereas graphs inferred by LiNGAM are not falsified using our metric at $\alpha=5\%$ (\tabref{tab:discovery_sachs}). Furthermore, a ranking based on our metric is in accordance with an SHD ranking (NOTEARS $>$ CAM $>$ LiNGAM). Both inequalities are significant with $p<0.001$ for SHD and $p_\text{LMC}$ when tested using a Wilcoxon signed-rank test. For further experimental results see \suppref{sec:app_results_causal_discovery}.

\begin{table}[tb]
    \centering
    \caption{\label{tab:runtimes_large_graphs} Runtime  of $p_\text{LMC}$ for large graphs with up to 200 nodes. All graphs were modeled as ER-$n$-1, $n\in\{10,50,100,200\}$. Data were generated using nonlinear conditionals and $N=1000$ samples. For each test we sample $100$ permutations, sufficient to reject the null at $\alpha=1\%$. As CI test we employed the GCM with boosted decision trees as regressor. See \suppref{sec:supp_runtime} for more details.}
    \fontsize{9pt}{9pt}\selectfont
    \setlength{\tabcolsep}{3.5pt}
    \begin{tabular}{lcccc}\toprule
\textbf{Nodes $n$} & 10 & 50 & 100 & 200\\ \midrule
\textbf{Runtime [s]} & 5 $\pm$ 3 & 97 $\pm$ 7 & 507 $\pm$ 18 & 3,386 $\pm$ 87 \\ \bottomrule
\end{tabular}

    \vspace{-\baselineskip}
\end{table}
\subsection{Runtime}\label{sec:exp_runtime}
\looseness=-1 Large graphs may entail thousands of CIs that need to be tested in order to compute our metric. Therefore, we evaluated the feasibility of applying $p_\text{LMC}$ to graphs with up to 200 nodes (\tabref{tab:runtimes_large_graphs}). We find that runtimes are reasonably fast ($<$ 1h) even for very large graphs. Note, that domain experts and causal discovery methods would likely take much longer to come up with a DAG $\Ggiven$ in the first place. \Eg, on DAGs with 50 and 100 nodes CAM and NOTEARS are about an order of magnitude slower, respectively~\citep{rolland2022,lachapelle2019}. More information on the algorithmic complexity of our metric is provided in \suppref{sec:app_algorithms}.

\section{Discussion}\label{sec:discussion}
In this work we addressed the lack of a suitable metric to evaluate an estimated DAG on observed data. To this end we discussed an existing absolute metric that, without a baseline comparison, is difficult to interpret. 

We defined a set of properties that a suitable baseline should satisfy and found that sampling random node permutations of the given DAG is a natural way to satisfy these requirements. Using this baseline, we derived a novel metric which comprises two tests that measure first, how characteristic a given graph is in the sense that it is falsifiable by testing CIs and second, whether the given graph is significantly better than a random one in terms of CIs. Using graphs originating from emulated domain experts and causal discovery algorithms we evaluated our method on two types of data. Synthetic data with known true DAG and real-world data with a consensus graph established in the literature. Furthermore, we contributed a novel data set from cloud monitoring where an estimate of the true causal graph is given by the inverted dependency graph. 

We acknowledge that our baseline may seem like a low bar to clear and other baselines could be considered in future work. However, we argue that besides being a natural choice that satisfies a number of desirable properties (P1 -- P3 in \secref{sec:node_permutation}) our experiments show that `better than random' is a surprisingly high threshold that simulated domain experts and causal discovery algorithms fail to meet in many settings. We also note that applying our algorithm to DAGs inferred by causal discovery algorithms should be done with caution as elaborated in \secref{sec:experimental_setup} and other approaches may be better suited for this purpose \cite{faller2024}. Furthermore, in certain applications, specific suggestions for local edge improvements that go beyond the simple report of the triplets that result in LMC violations may be desired, and we leave this interesting direction for future work. Another promising direction could be to extend our node-permutation baseline to the likelihood $p(\data | \Ggiven)$ (\cf \secref{sec:related_work}).

\bibliography{bibliography.bib}
\newpage

\makeatletter
\renewcommand\section{\@startsection {section}{1}{\z@}{-2.0ex plus
-0.5ex minus -.2ex}{3pt plus 2pt minus 1pt}{\Large\bf\raggedright}}
\makeatother

\appendix
\onecolumn
\thispagestyle{plain}

\renewcommand{\thetable}{A\arabic{table}}
\renewcommand{\thefigure}{A\arabic{figure}}
\setcounter{table}{0}
\setcounter{figure}{0}
\setcounter{section}{0}

\hsize\textwidth
{\centering
{\LARGE\bfseries Toward Falsifying Causal Graphs Using a Permutation-Based Test \\ Supplementary Materials \par}} \vskip0.2in

\section{Experiments}\label{sec:app_experiments}
\subsection{False positive rate of different CI tests}\label{sec:app_type_1_ci}
In \figref{fig:type1_ci} we show the probability of type I errors for different sizes $D$ of the conditioning set $\{Z_1, Z_2, ..., Z_D\}$ for one parametric (partial correlation) and two nonparametric CI tests, the \textit{Kernel-based Conditional Independence Test (KCI)} \citep{zhang2011} and a test based on the \textit{Generalized Covariance Measure (GCM)} \citep{shah2020}. In our experiment, $Z_i$ were \iid standard Gaussian and $X$ and $Y$ were generated from $Z_1$ alone via $\beta Z + N$ with $\beta \in \text{U}(-1, 1), N\sim \mathcal{N}(0, \sigma^2), \sigma=0.1$ independent across $X, Y$. We used GCM with boosted regression trees as the regression method (as in all other experiments in this work) and neither learned the hyperparameters for GCM nor for KCI for comparability. The reported FPR is over 1000 random replications.

More closely related to our problem setup we can also investigate the fraction of LMC violations (caused by type I errors) for some DAG. To this end, we generated data with $N=400$ from ER-$n$-$d$ graphs with nonlinear conditionals. LMC violations were tested (with $\alpha=5\%$) using KCI. In \figref{fig:lmc_violations_correction}; left, we find an increasing FPR with increasing graph size. This cannot be corrected with standard methods for multiple testing adjustments (\figref{fig:lmc_violations_correction}; center, right), where we still observe high FPR of up to 20\% for the true graph in some settings. Thus, without a baseline comparison we cannot judge whether some observed fraction of violations is high or low.

\begin{figure}[b]
    \centering
    \includegraphics[width=0.6\textwidth]{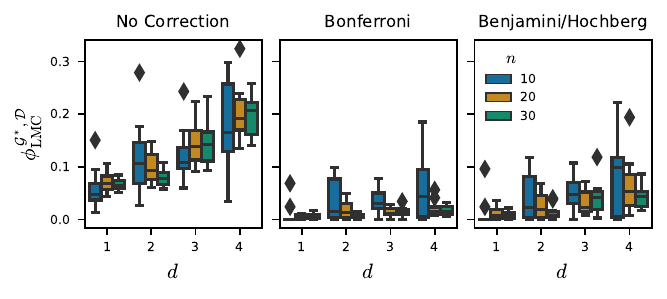}
    \vspace{-\baselineskip}
    \caption{Fraction of LMC violations for increasing number of nodes / connectivity. We generated data with $N=400$ from ER-$n$-$d$ graphs with nonlinear conditionals. LMC violations were tested (with $\alpha=5\%$) using KCI. We applied either no, Bonferroni, or Benjamin/Hochberg correction.}\vskip-\baselineskip
    \label{fig:lmc_violations_correction}
\end{figure}
\subsection{Node domain expert}\label{sec:node_expert}
Here, we outline the procedure by which the given DAG by a node domain expert \figref{fig:domain_expert_nodes} is constructed from the true DAG $\Gtrue$:
\begin{enumerate}[label=(\arabic*)]
    \item Extracting a random subset $K\subseteq \verts$ of the nodes (these are the nodes of which the domain expert has complete knowledge), for which
    \[
    \forall i, \forall j\in K: (i, j) \in \Etrue \Leftrightarrow (i, j) \in \Egiven~\text{, and}
    \]
    \item all other edges $(i, j) \in \Etrue, ~i\notin K \lor j \notin K$ are randomly shuffled while ensuring that $\Ggiven$ remains acyclic.
\end{enumerate}

\subsection{Edge domain expert}\label{sec:edge_expert}
Starting from the true DAG $\Gtrue$, to construct $\Ggiven$, some edges $N, N\cap\Etrue=\emptyset$ are added, $M\subseteq\Etrue$ are removed, and $L \subseteq \Etrue \setminus M$ are flipped. Similar to DE-$\verts$ we further ensure that $\Ggiven$ has the same sparsity as $\Gtrue$ by enforcing $|N| = |M|$. Using $\edges^f = \{(j, i): (i, j) \in \Etrue\}$ to denote the set of flipped edges, the given DAG $\Ggiven=(\verts, \Egiven)$ is constructed by:
\begin{enumerate}[label=(\arabic*)]
    \item Adding a random set of edges $N\subset \verts^2 \setminus (\Etrue \cup \edges^f)$
    \[
        \forall (i,j) \in N: (i,j) \not\in \Etrue \Leftrightarrow (i,j) \in \Egiven~\text{, and}
    \]
    \item removing a random subset $M\subset \Etrue$ of the edges:
    \[
        \forall (i,j) \in M: (i,j) \in \Etrue \Leftrightarrow (i,j) \not\in \Egiven~\text{, and}
    \]
    \item flipping a random subset $L \subset \Etrue \setminus M$ of the edges:
    \[
        \forall (i,j) \in L: (i,j) \in \Etrue \Leftrightarrow (j,i) \in \Egiven~.
    \]
\end{enumerate}
By construction $\Ggiven$ of a DE-$\edges$ is related to the Structural Hamming Distance (SHD)\footnote{We here consider the version of SHD in which anticausal edges are counted once.} \citep{acid2003,tsamardinos2006}, via $\text{SHD}(\Ggiven, \Gtrue) = |N| + |M| + |L|$, and thus the desired similarity of a given graph can be controlled by means of the SHD. We characterize different DE-$\edges$ by the $\text{SHD}(\Ggiven, \Gtrue)$ they entail (or $\text{SHD}(\Ggiven, \Gtrue) / |\Etrue|$ to compare systems with different sparsity), where $\text{SHD}(\Ggiven, \Gtrue) / |\Etrue| = 0$ corresponds to $\Ggiven=\Gtrue$.
\begin{figure}[tb]
    \centering
    \begin{subfigure}{0.31\textwidth}
\centering
\begin{tikzpicture}[node distance=0.8cm,on grid]
    \node[] at (0, 1) {$\Gtrue$};
    \node[state] (x1) at (0,0) {$X_1$};
    \node[state] (x2) at (-0.7, -1) {$X_2$};
    \node[state] (x3) at (0.7, -1) {$X_3$};
    \node[state] (x4) at (0, -2) {$X_4$};

    \path (x1) edge (x2);
    \path (x2) edge (x3);
    \path (x2) edge (x4);
    \path (x3) edge (x4);
\end{tikzpicture}
\caption{}\label{fig:true_dag}
\end{subfigure}
\begin{subfigure}{0.31\textwidth}
\centering
\begin{tikzpicture}[node distance=0.8cm,on grid]
\node[] at (0, 1) {$\Ggiven_{\text{DE-}\verts}$};
    \node[state] (x1) at (0,0) {$X_1$};
    \node[state] (x2) at (-0.7, -1) {$X_2$};
    \node[state, line width=0.9pt] (x3) at (0.7, -1) {$X_3$};
    \node[state, line width=0.9pt] (x4) at (0, -2) {$X_4$};

    \path (x3) edge[dashed] (x1);
    \path (x3) edge[dashed] (x2);
    \path (x2) edge[dashed] (x4);
    \path (x3) edge[line width=0.9pt] (x4);
\end{tikzpicture}
\caption{}\label{fig:domain_expert_nodes}
\end{subfigure}
\begin{subfigure}{0.31\textwidth}
\centering
\begin{tikzpicture}[node distance=0.8cm,on grid]
\node[] at (0, 1) {$\Ggiven_{\text{DE-}\edges}$};
    \node[state] (x1) at (0,0) {$X_1$};
    \node[state] (x2) at (-0.7, -1) {$X_2$};
    \node[state] (x3) at (0.7, -1) {$X_3$};
    \node[state] (x4) at (0, -2) {$X_4$};

    \path (x1) edge[dashed] (x3);
    \path (x3) edge[dotted] (x2);
    \path (x2) edge[line width=0.9pt] (x4);
    \path (x3) edge[line width=0.9pt] (x4);
\end{tikzpicture}
\caption{}\label{fig:domain_expert_edges}
\end{subfigure}
    \caption{\label{fig:domain_experts} (\subref{fig:true_dag}) True DAG $\Gtrue$; (\subref{fig:domain_expert_nodes}) $\Ggiven$ from a DE-$\verts$, where $K = \{X_3, X_4\}$ (\protect\inlinearc{2em}{line width=0.9pt}) and the remaining 3 edges are randomly shuffled (\protect\inlinearc{2em}{dashed}); (\subref{fig:domain_expert_edges}) $\Ggiven$ from a DE-$\edges$, where $N = \{(X_1, X_3)\}$ (\protect\inlinearc{2em}{dashed}), $M = \{(X_1, X_2)\}$, and $L=\{(X_2, X_3)\}$ (\protect\inlinearc{2em}{dotted}).}
\vskip -\baselineskip
\end{figure}
\subsection{Synthetic data}\label{sec:app_experiments_synthetic}
For all experiments on synthetic data we sample $T=\num{e3}$ node permutations and use datasets with $N=\num{e3}$ observations. For experiments on graphs inferred via causal discovery algorithms we infer the graph on $N=\num{e3}$ independent samples. To investigate the effect of $N$ and $T$ on $p_\text{LMC}$ we run ablation studies with $N,T \in \{\num{e1}, \num{e2}, \num{e3}, \num{e4}\}$.
\paragraph{Generation of Random DAGs} To sample random DAGs we first generate random graphs under the Erd\H{o}s-R\'{e}nyi model \citep{erdos1959} and convert a graph to a DAG by only keeping the lower triangle of its adjacency matrix.
\paragraph{Linear Mechanisms}For simulating linear relationships between nodes, we model conditionals as ANMs
\begin{align}
    X_i &= f_{i}(\Pa{i}^{\graph^t}) + N_i \label{eq:ANM}\\ 
    f_{i} &= \sum_{j=0}^M w_j \left(\Pa{i}^{\graph^t}\right)_j,~N_i\sim \mathcal{N}(0, \sigma^2), \sigma=0.1, ~
\end{align}
where $M=|\Pa{X_i}^{\graph^t}|$, and $w_j \sim \text{U}(-1.0, 1.0)$.
For all synthetic data sets with linear meachnisms we utilize a CI test of partial correlation to evaluate our metric.
\paragraph{Nonlinear Mechanisms} For simulating nonlinear relationships between nodes we use 3-layer MLPs with randomly initialized weights to generate functions $f_i$ in \eqref{eq:ANM}:
\begin{align}
    f_{i}(\pa{i}{\Gtrue}) = \sigma \left ( \sum_{l=0}^O w_{1l}^{(3)} \sigma\left( \sum_{k=0}^N w_{lk}^{(2)} \sigma\left( \sum_{j=0}^M w_{kj}^{(1)} \left(\pa{i}{\Gtrue}\right)_j\right)\right)\right)~,
\end{align}
where $\sigma$ denotes the sigmoid function, $N, O \overset{\text{i.i.d.}}{\sim} \text{U}(2, 100)$, and entries in the weight matrices $w_{ij}^{(k)} \sim \text{U}(-5.0, 5.0)$.
For all synthetic data sets with nonlinear meachnisms we utilize a CI test based on the generalised covariance measure \citep{shah2020} with boosted decision trees as the regression model.
\subsection{Real data}\label{sec:app_experiments_real}
The confidence intervals in \tabref{tab:real_results} are computed via
\begin{align}
    \text{confidence interval} = \hat{p}_\text{LMC} \pm z \sqrt{\frac{\hat{p}_\text{LMC}(1-\hat{p}_\text{LMC})}{T}} ~, \label{eq:ci_p_lmc}
\end{align}
where $\hat{p}_\text{LMC}$ denotes a p-value estimated using $T$ permutations and $z$ denotes the $z$-score. In all experiments on real-world data we set $T=1000$.

For the \textbf{Protein Signaling Network} dataset, we define the true DAG (\figref{fig:sachs_dag}) according to the ``extended expert model'' described in \cite{sachs2005,ramsey2018,glymour2019}. A pairwise scatter plot and histograms of the 11 variables is provided in \figref{fig:sachs_data}.
Because of the relatively large number of samples, we validate LMCs using a CI test based on the generalised covariance measure \citep{shah2020} with boosted decision trees as the regression model.

For the \textbf{Auto MPG} dataset, we define the true DAG (\figref{fig:auto_dag}) based on known causal relationships and similar to the ones described in \citet{wang2017,teshima2021}. In \figref{fig:auto_data} we provide a pairwise scatter plot and histograms of the 6 variables we used to model the system. Due to the discrete features we employ a kernel-based CI test \citep{zhang2011} to validate LMCs.

For the \textbf{Application Performance Monitoring} dataset, we employ the \texttt{PetAdoptions} application as introduced by AWS in an online workshop\footnote{\url{https://catalog.workshops.aws/observability/}} and available on GitHub\footnote{\url{https://github.com/aws-samples/one-observability-demo/tree/main/PetAdoptions}}. We ran the application for six days and collected latency data of each microservice averaged over \SI{20}{\min}, leading to a total of $N=432$ datapoints which we treat as \iid for our following experiments (see \figref{fig:apm_data} for a pairwise scatter plot and histogram of the variables). In our experiments, we assume the transpose of the dependency graph of the application to be the causal DAG of the system \figref{fig:apm_dag}. Note, that for some of the nodes no data was recorded, or the data had missing values (due to less frequent calls to the respective microservice). In our experiments we excluded those nodes from the graph evaluation.
Because of the large number of graph-implied independencies  we employ a CI test based on the generalized covariance measure using boosted decision trees as regression model.

\begin{figure}[tb]
    \centering
    \vspace*{-\baselineskip}
    \includegraphics[width=0.51\textwidth]{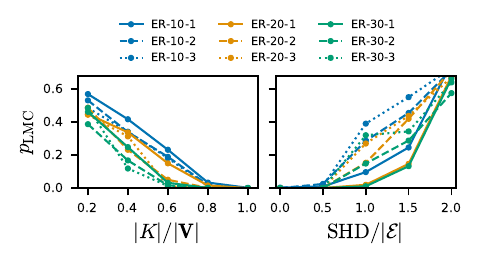}
    \caption{\label{fig:p_lmc_synthetic_linear} $p_\text{LMC}$ for random ER graphs with Gaussian-linear mechanisms.}
    \end{figure}
\begin{figure}[htb]
    \centering
    \begin{subfigure}{0.3\textwidth}
        \includegraphics[width=\textwidth]{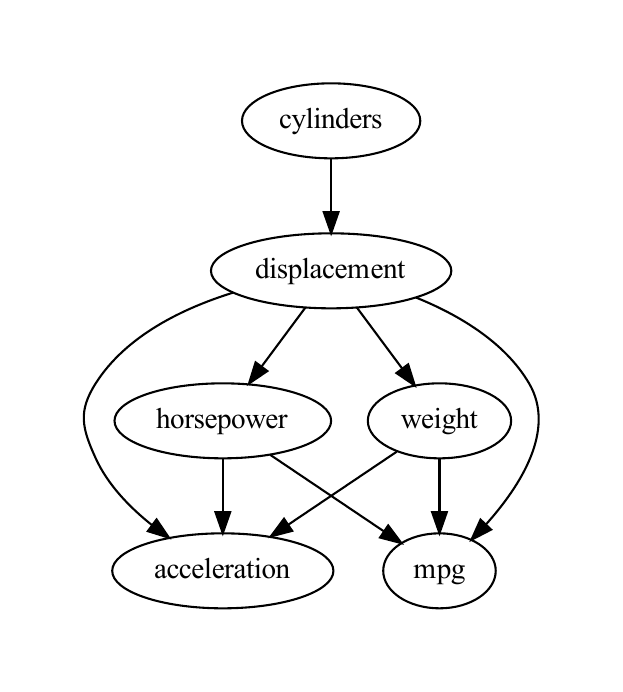}
        \caption{}\label{fig:auto_dag}
    \end{subfigure}
    \begin{subfigure}{0.3\textwidth}
        \includegraphics[width=\textwidth]{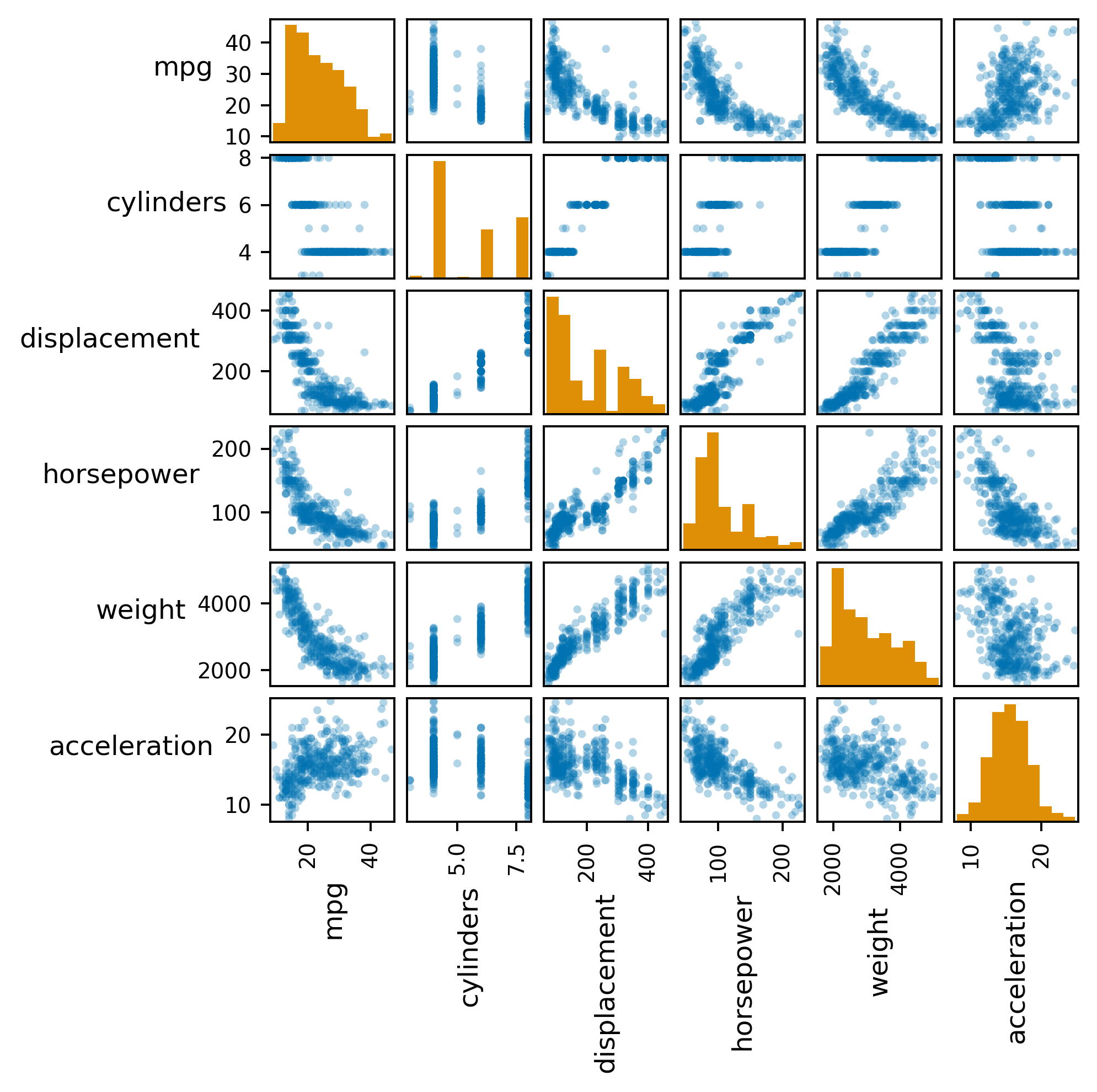}
        \caption{}\label{fig:auto_data}
    \end{subfigure}\newline
    \begin{subfigure}{0.3\textwidth}
        \includegraphics[width=\textwidth]{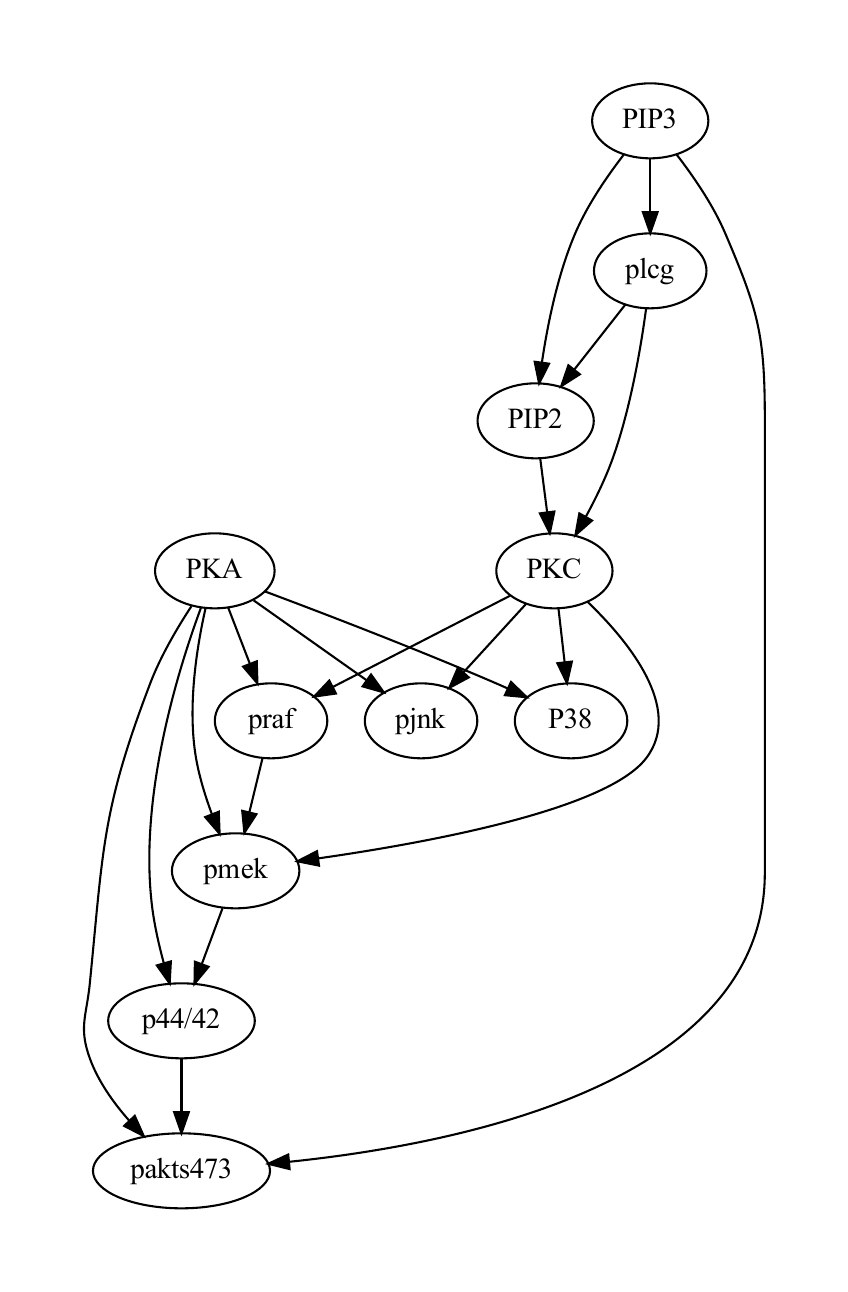}
        \caption{}\label{fig:sachs_dag}
    \end{subfigure}
    \begin{subfigure}{0.4\textwidth}
        \includegraphics[width=\textwidth]{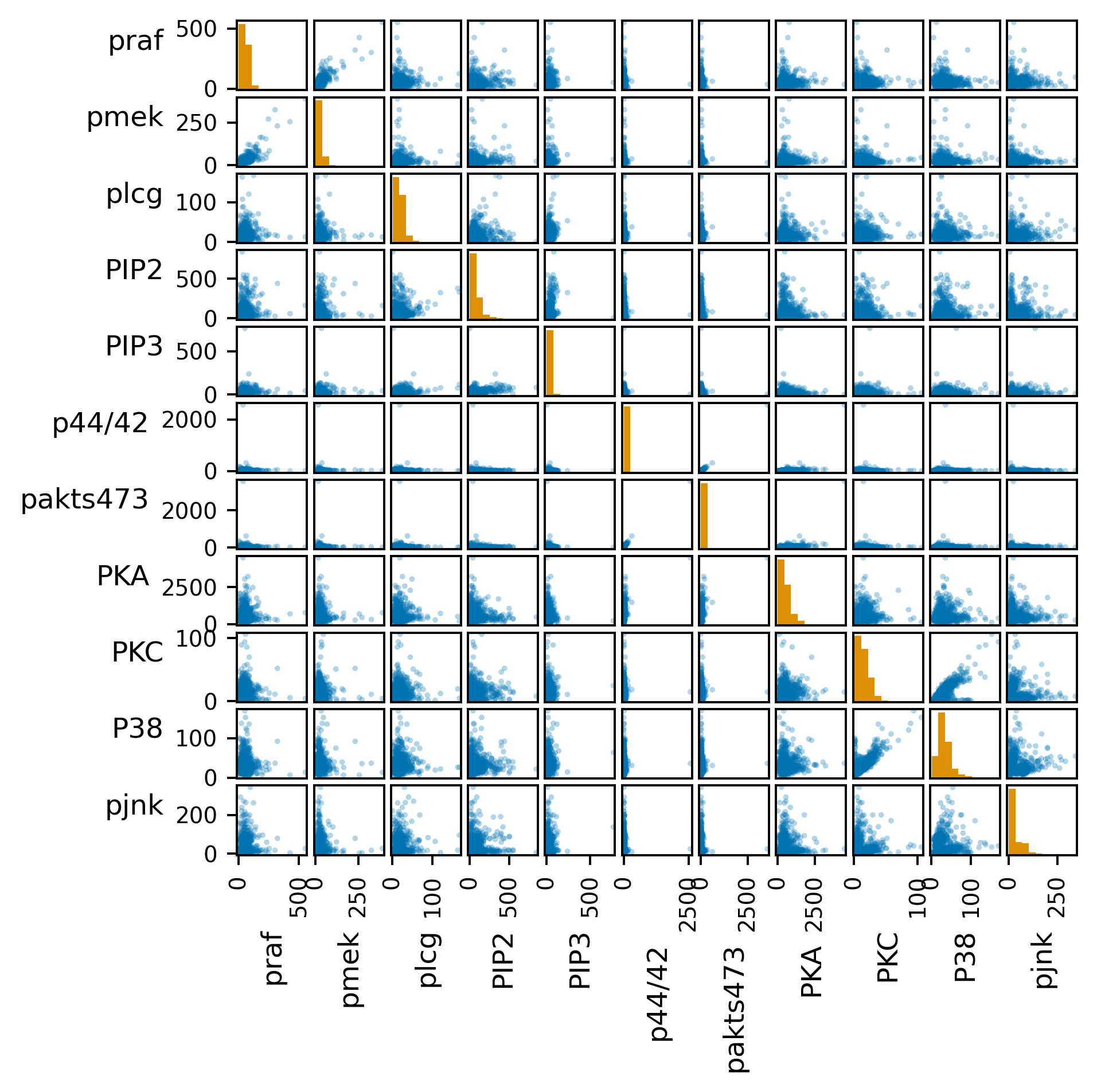}
        \caption{}\label{fig:sachs_data}
    \end{subfigure}\newline
    \begin{subfigure}{0.5\textwidth}
        \includegraphics[width=\textwidth]{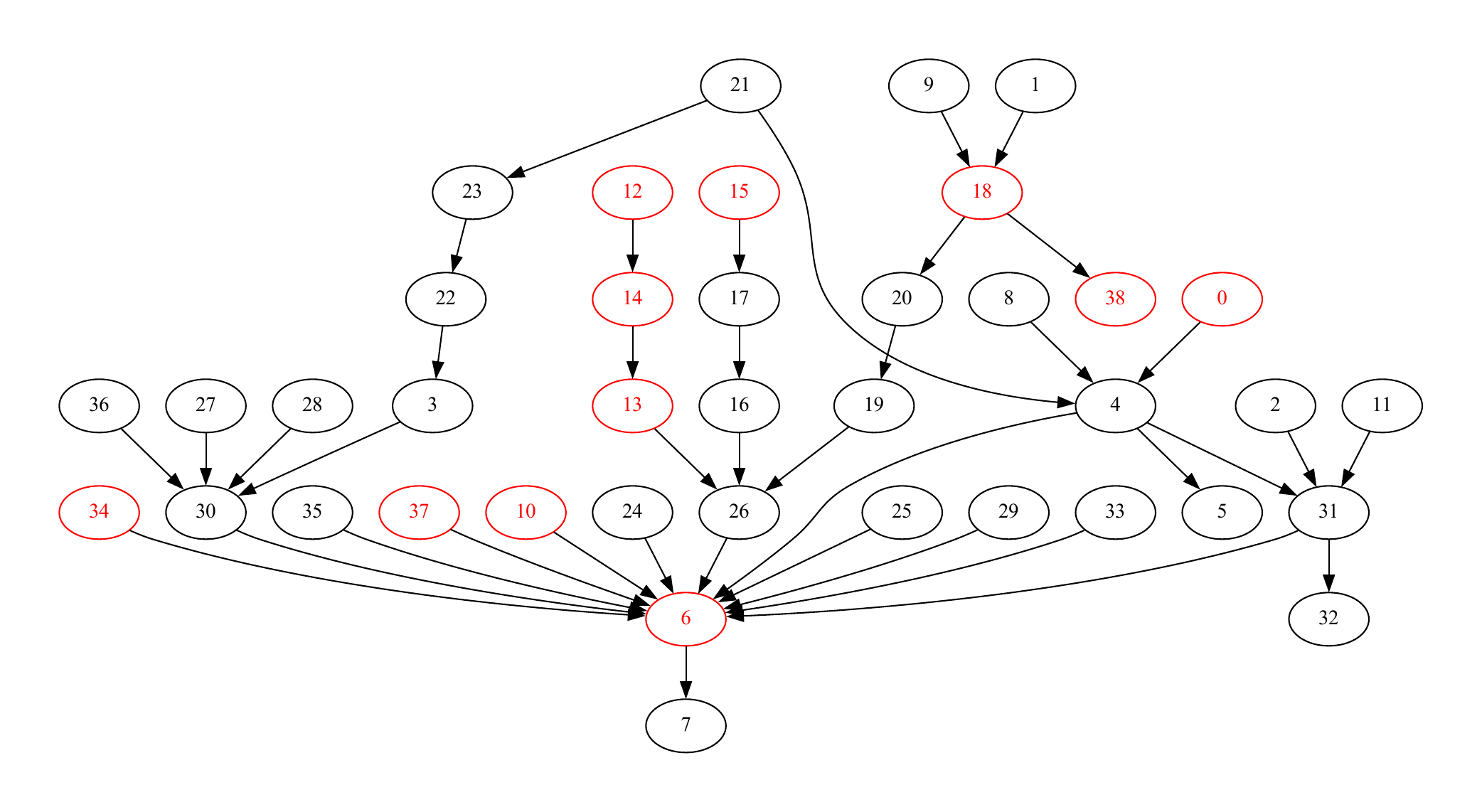}
        \caption{}\label{fig:apm_dag}
    \end{subfigure}
    \begin{subfigure}{0.3\textwidth}
        \includegraphics[width=\textwidth]{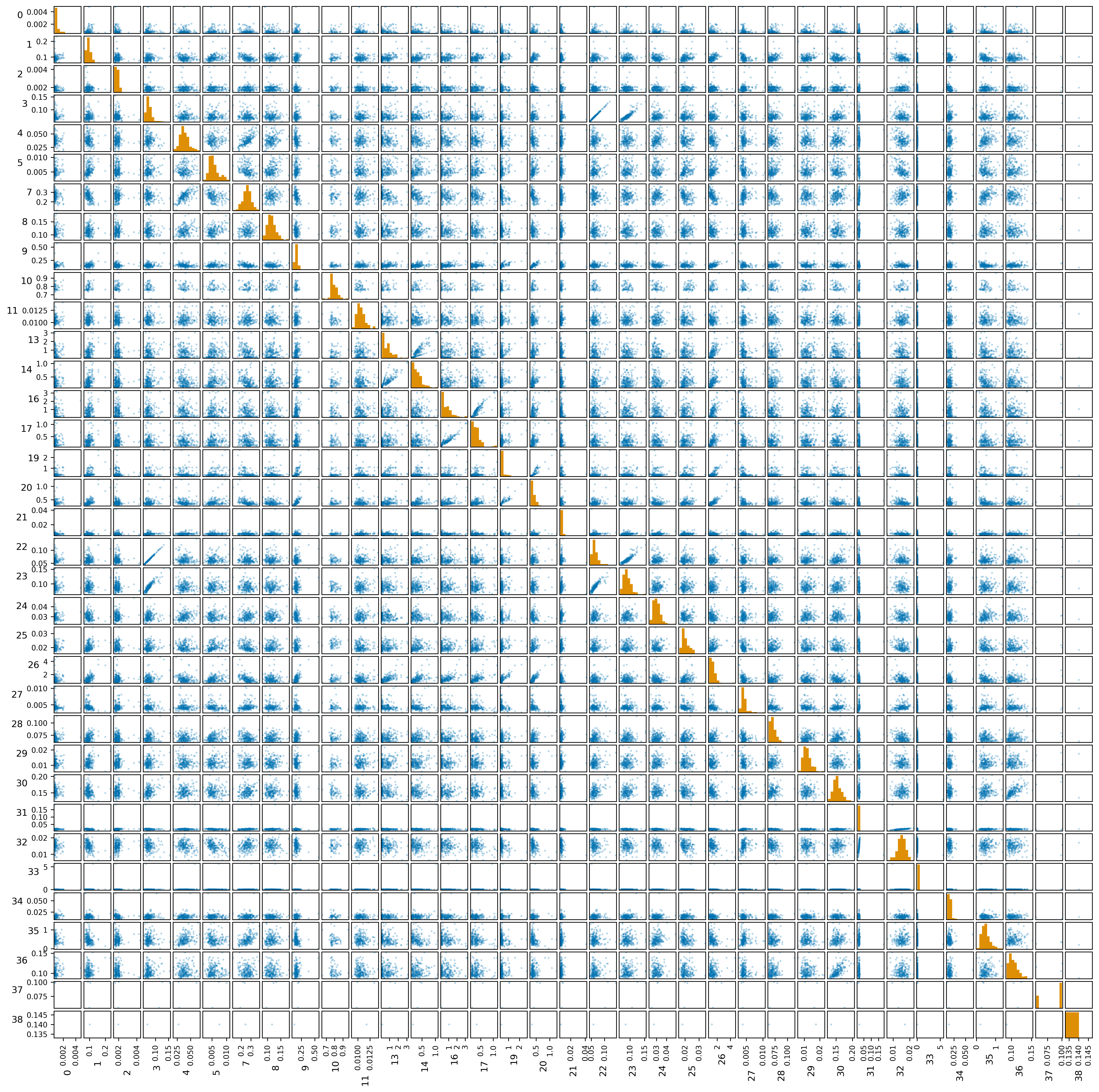}
        \caption{}\label{fig:apm_data}
    \end{subfigure}
    \caption{\label{fig:real_dag_data}(\subref{fig:auto_dag}) $\graph^t$ for the Auto MPG data; (\subref{fig:auto_data}) Pairwise scatter plot and histograms of the variables in the Auto MPG data from \citet{quinlan1993}; (\subref{fig:sachs_dag}) $\graph^t$ for the protein signaling networks data; (\subref{fig:sachs_data}) Pairwise scatter plot and histograms of the 11 variables in the protein signaling network data from \citet{sachs2005}; (\subref{fig:apm_dag}) $\graph^t$ for the APM data with excluded nodes marked in red; (\subref{fig:apm_data}) Pairwise scatter plot and histograms of the variables in the APM data from.}
\end{figure}
\subsection{Runtime}\label{sec:supp_runtime}
In \tabref{tab:runtimes_large_graphs} we report the runtime of $p_\text{LMC}$ for larger graphs with up to 200 nodes. All graphs were modeled as ER-$n$-1, $n\in \{ 10,50,100,200 \}$. Data were generated using nonlinear conditionals and $N=1000$ samples. For each test we sample $100$ permutations, sufficient to reject the null at $\alpha=0.01$. As CI test we employed the Generalized Covariance Measure (GCM) with boosted decision trees as regressor. We report mean $\pm$ std over five random ER graphs. The time measurements were taken on a machine with an AMD Ryzen Threadripper 3960X CPU and 128 GB of RAM.

\section{Algorithmic complexity}\label{sec:app_algorithms}
Recall that $T$ is the number of permutations we sample. For each permutation we iterate through the nodes and check conditional independence of the node and each non-descendent conditioned on the parents. 
For $n$ nodes in the graph we thus require $\mathcal{O}(T\cdot n^2)$ independence test. The complexity of each such test further depends on the degree of the nodes that represent the conditioning set. 
A practical improvement in case the number of violations is small is to start out with testing for each node $i$ whether
$X_i \CI_\data \nd{i}{\graph} \setminus\pa{i}{\graph} \mid \pa{i}{\graph}$. If yes we can move on to the next node. If not, then we have to do some more work to find out exactly how many nondescendents are dependent. This can be done by recursively splitting the set of nondescendents for which a violation was found and recursing on those with violations. That number of group conditional independence test is $\mathcal{O}(|\vlmc[\graph, P] | \cdot \log(n^2))$. It is worth pointing out though that the complexity of the test depends on the group size. So the time savings do not just depend on the number of violations but also the exact test that is used. Reducing the number of tests not only reduces the running time but also the overall type I error. Alternatively, we can consider sampling from $\tpa{\Ggiven}$ to estimate the number of violations rather than computing it exactly.

\section{Additional experimental results}\label{sec:app_results}
In the following we present further experimental results for both synthetic (\suppref{sec:app_results_synthetic}) and real data (\suppref{sec:app_results_synthetic}), as well as on graphs inferred via causal discovery algorithms (\suppref{sec:app_results_causal_discovery}).
\subsection{Synthetic data}\label{sec:app_results_synthetic}
\textbf{Gaussian-Linear Mechanisms} Figure \ref{fig:p_lmc_synthetic_linear} shows $p_\text{LMC}$ (\cf \secref{sec:node_permutation}) for synthetic graphs of various size and sparsity and with linear-gaussian mechanisms, where CIs were tested via partial correlation. As expected, we find that the average $p_\text{LMC}$ monotonically decreases with increasing amount of domain knowledge for both models of domain experts. When the domain expert has complete knowledge of the underlying system ($\Ggiven = \Gtrue$, corresponding to $|K|/|\verts| = 1$ for DE-$\verts$ and $\text{SHD}/|\edges|=0$ for DE-$\edges$), we reject the hypothesis that the DAG is as bad as a random node permutation with significance level $\alpha=1\%$ for all configurations.

\textbf{Effect of number of sampled permutations and number of observations} In \tabref{tab:num_perms_samples} we report results on the effect of the number of observations and the number of sampled permutations on our proposed metric for ER graphs with nonlinear mechanisms and of different size and sparsity. In \figrefs{fig:hists_synthetic_linear,fig:hists_synthetic_nonlinear}, we show the accumulated histograms of fractions of violations for both domain experts on random ER graphs with varying density and size for linear and nonlinear mechansims, respectively.

\begin{table}[tb]
\centering
 \caption{\label{tab:num_perms_samples}Effect of number of observations $N$ and number of sampled permutations $T$ on our metric. We evaluate $p_\text{LMC}$ of randomly sampled ER graphs with varying number of nodes $n$ and sparsity $d$ for which data is generated via nonlinear SCMs (\cf \secref{sec:experiments_synthetic}). The reported numbers are averages over 50 random samples.}
 \fontsize{8pt}{8pt}\selectfont
 \setlength{\tabcolsep}{4pt}
\begin{tabular}{@{}ccS[scientific-notation = true, round-precision = 1, table-format = 1.1e-1]S[scientific-notation = true, round-precision = 1, table-format = 1.1e-1]S[scientific-notation = true, round-precision = 1, table-format = 1.1e-1]S[scientific-notation = true, round-precision = 1, table-format = 1.1e-1]S[scientific-notation = true, round-precision = 1, table-format = 1.1e-1]S[scientific-notation = true, round-precision = 1, table-format = 1.1e-1]S[scientific-notation = true, round-precision = 1, table-format = 1.1e-1]S[scientific-notation = true, round-precision = 1, table-format = 1.1e-1]@{}} 
    \toprule 
    &     & \multicolumn{4}{c}{\textbf{Number of permutations ($T$)}} & \multicolumn{4}{c}{\textbf{Number of observations ($N$)}} \\\cmidrule(lr){3-6}\cmidrule(lr){7-10}
$n$ & $d$ & {10} & {100} & {1000} & {10000} & {10} & {100} & {1000} & {10000} \\ \midrule
    &1   & 0.006 & 0.0094 & 0.00936 & 0.008902 & 0.1045 & 0.0084 & 0.0028 & 0.002 \\
10  &2   & 0 & 0 & 0 & 2.2e-05 & 0.08496 & 0.00036 & 2e-05 & 0 \\
    &3   & 0 & 0 & 0 & 0 & 0.1003 & 0.0004 & 2e-05 & 0 \\\midrule
    &1   & 0 & 0 & 4e-05 & 4.4e-05 & 0.00724 & 2e-05 & 0 & 0 \\
20  &2   & 0 & 0 & 0 & 0 & 0.00258 & 0 & 0 & 0 \\
    &3   & 0 & 0 & 0 & 0 & 0.0098 & 0 & 0 & 0 \\\midrule
    &1   & 0 & 0 & 0 & 0 & 0.00332 & 0.00052 & 0 & 0 \\
30  &2   & 0 & 0 & 0 & 0 & 0.00078 & 0 & 0 & 0 \\
    &3   & 0 & 0 & 0 & 0 & 0.0034 & 0 & 0 & 0 \\\bottomrule
\end{tabular}

\end{table}
\begin{figure}[!t]
\captionsetup[subfigure]{justification=justified,singlelinecheck=false}
\begin{subfigure}{\linewidth}
	\centering{\includegraphics[width=0.6\linewidth]{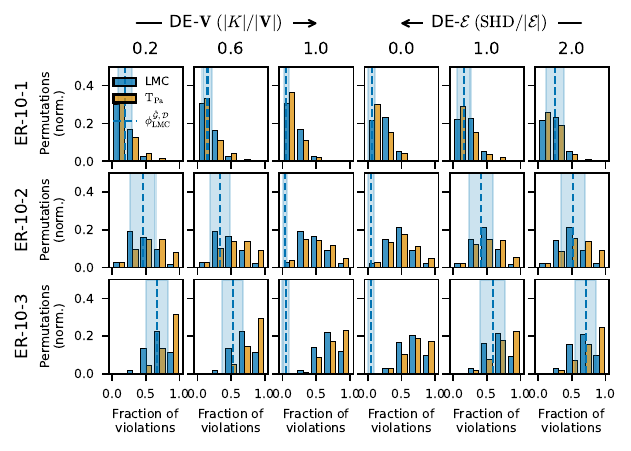}}
	\vskip -0.4\linewidth
	\caption{}
	\vskip 0.35\linewidth
	\label{fig:hists_synthetic_linear_ER-10}
\end{subfigure}
\begin{subfigure}{\linewidth}
	\centering{\includegraphics[width=0.6\linewidth]{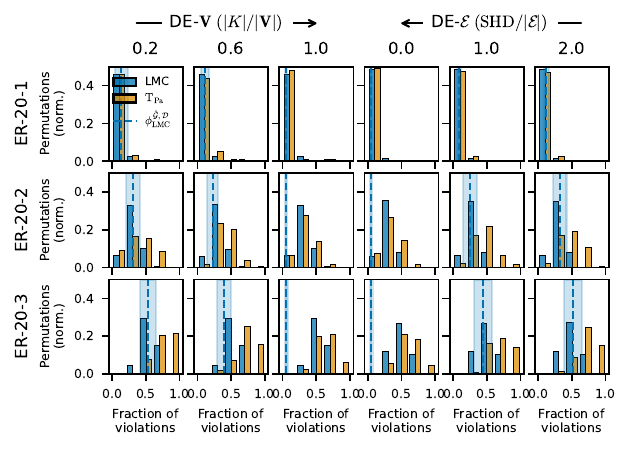}}
	\vskip -0.4\linewidth
	\caption{}
	\vskip 0.35\linewidth
	\label{fig:hists_synthetic_linear_ER-20}
\end{subfigure}
\begin{subfigure}{\linewidth}
	\centering{\includegraphics[width=0.6\linewidth]{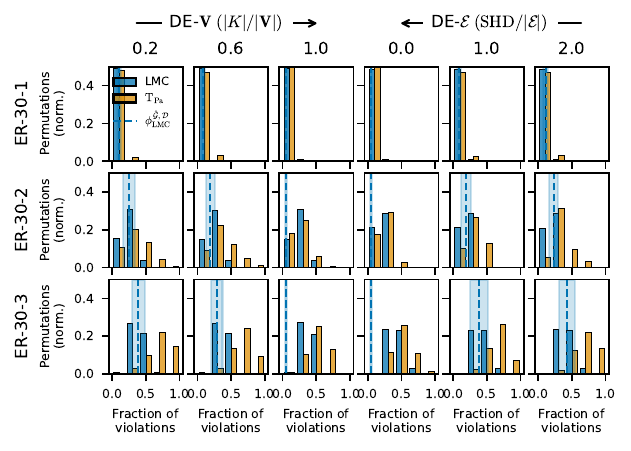}}
	\vskip -0.4\linewidth
	\caption{}
	\vskip 0.35\linewidth
	\label{fig:hists_synthetic_linear_ER-30}
\end{subfigure}
\caption{\label{fig:hists_synthetic_linear}Best viewed in color. Aggregated histograms of LMC (blue) and d-Separation (orange) violations for random ER DAGs with linear-gaussian mechanisms with $n=10$ (\subref{fig:hists_synthetic_linear_ER-10}), $n=20$ (\subref{fig:hists_synthetic_linear_ER-20}), and $n=30$ (\subref{fig:hists_synthetic_linear_ER-30}). The arrows indicate the direction of increasing domain knowledge for the two different domain experts. Additionally, we provide the mean $\flmc[\Ggiven, \data]$ over 50 sampled DAGs (dashed) and standard deviation (shaded).}
\vskip -\baselineskip
\end{figure}
\begin{figure}[!t]
\captionsetup[subfigure]{justification=justified,singlelinecheck=false}
\begin{subfigure}{\linewidth}
	\centering{\includegraphics[width=0.6\linewidth]{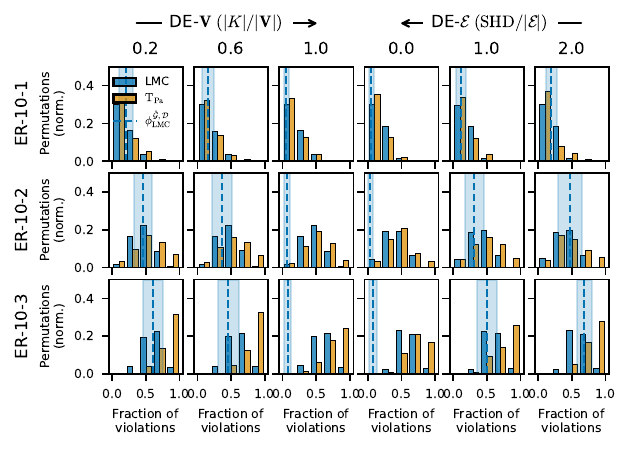}}
	\vskip -0.4\linewidth
	\caption{}
	\vskip 0.35\linewidth
	\label{fig:hists_synthetic_nonlinear_ER-10}
\end{subfigure}
\begin{subfigure}{\linewidth}
	\centering{\includegraphics[width=0.6\linewidth]{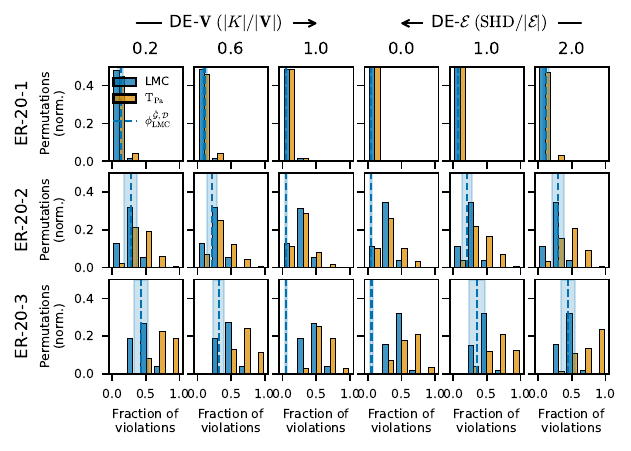}}
	\vskip -0.4\linewidth
	\caption{}
	\vskip 0.35\linewidth
	\label{fig:hists_synthetic_nonlinear_ER-20}
\end{subfigure}
\begin{subfigure}{\linewidth}
	\centering{\includegraphics[width=0.6\linewidth]{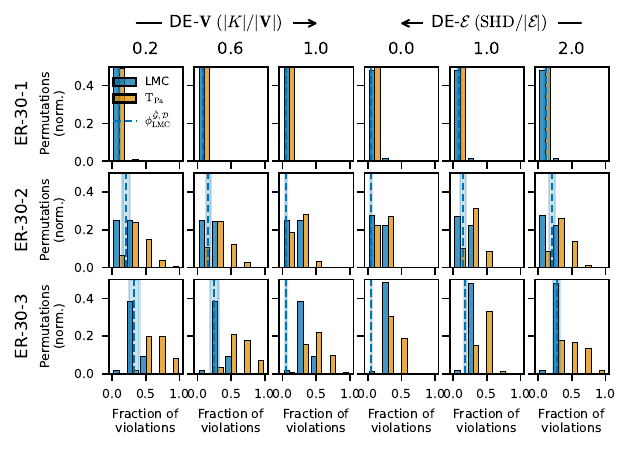}}
	\vskip -0.4\linewidth
	\caption{}
	\vskip 0.35\linewidth
	\label{fig:hists_synthetic_nonlinear_ER-30}
\end{subfigure}
\caption{\label{fig:hists_synthetic_nonlinear}Best viewed in color. Aggregated histograms of LMC (blue) and d-Separation (orange) violations for random ER DAGs with nonlinear mechanisms with $n=10$ (\subref{fig:hists_synthetic_nonlinear_ER-10}), $n=20$ (\subref{fig:hists_synthetic_nonlinear_ER-20}), and $n=30$ (\subref{fig:hists_synthetic_nonlinear_ER-30}). The arrows indicate the direction of increasing domain knowledge for the two different domain experts. Additionally, we provide the mean $\flmc[\Ggiven, \data]$ over 50 sampled DAGs (dashed) and standard deviation (shaded).}
\vskip -\baselineskip
\end{figure}
\subsection{Real data}\label{sec:app_results_real}
In \figref{fig:hists_real} we provide histograms of the number of fractions of LMC violations for the three real-world datasets.
\begin{figure}
    \centering
    \includegraphics[width=0.7\textwidth]{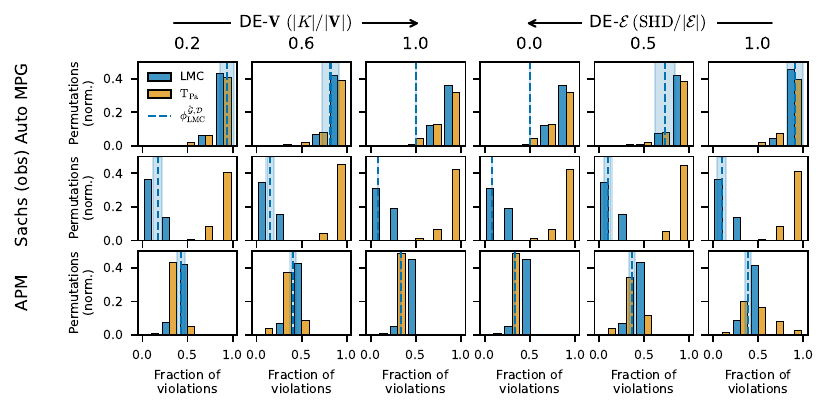}
    \caption{\label{fig:hists_real} Best viewed in color. Aggregated histograms of LMC (blue) and d-Separation (orange) violations for the Auto MPG \citep{quinlan1993}, Sachs \citep{sachs2005}, and APM data. The arrows indicate the direction of increasing domain knowledge for the two different domain experts. Additionally, we provide the mean $\flmc[\Ggiven, \data]$ over 50 sampled DAGs (dashed) and standard deviation (shaded).}
    \vskip -\baselineskip
\end{figure}

\subsection{Causal discovery algorithms}\label{sec:app_results_causal_discovery}
This work is mainly concerned with evaluating user-given graphs, \ie graphs originating from a domain expert that has knowledge about some of the variables/edges of the underlying causal system. However, if such domain knowledge is unavailable, practitioners may infer the graph from observational data using causal discovery. To test the applicability of our metric to such graphs (for arguments on why this applicability may be limited, \cf \secref{sec:experimental_setup}), we conduct additional experiments using graphs inferred via LiNGAM~\citep{shimizu2006}, CAM~\citep{buhlmann2014}, and NOTEARS~\citep{zheng2018} on synthetic $\Gtrue$ from ER-$n$-$d$, $n \in \{10,20,30\}, d\in \{1, 2, 3\}$ with nonlinear mechanisms ($N=2000$ samples)\footnote{Note that some causal discovery algorithms will perform better on our synthetic data and the motivation of these experiments is not to evaluate causal discovery algorithms.} and the protein signaling network dataset \citep{sachs2005}.\footnote{For LiNGAM and CAM, we use implementations from the Causal Discovery Toolbox (\url{https://github.com/FenTechSolutions/CausalDiscoveryToolbox}). For NOTEARS we use the authors implementation available at \url{https://github.com/xunzheng/notears}.} We randomly choose 50\% of the samples to infer the graph and the remaining samples to evaluate using our metric (\figref{fig:results_causal_discovery}). We repeated experiments (inferring using causal discovery + evaluation with our metric) ten times for each configuration of the synthetic data and 50 times for the protein signaling network dataset.

\begin{table}[tb]
    \centering
    \setlength{\tabcolsep}{2pt}
    \caption{\label{tab:discovery_synthetic} $p_\text{LMC}$ and SHD for graphs inferred by causal discovery on synthetic data with nonlinear mechanisms.}
    \fontsize{8pt}{12pt}\selectfont
    \begin{tabular}{@{}l
S[table-format=1.3(1)]
S[table-format=1.4(1)]
S[table-format=1.3(1)]@{}} 
\toprule                                & {\textbf{NOTEARS}}       & {\textbf{LiNGAM}}       & {\textbf{CAM}}      \\ \midrule
$p_\text{LMC}^{\Ggiven,\data}$ & 0.014 \pm 0.044        & 0.0048 \pm 0.024      & 0.031 \pm 0.11 \\
$\text{SHD}/|\edges|$    & 1.7 \pm 0.71        & 1.5 \pm 0.54        & 3.8 \pm 4.3    \\\bottomrule

\end{tabular}

    \vskip -\baselineskip
\end{table}
On synthetic data (\figref{fig:cd_p_lmc_synthetic_nonlinear}) we find that our metric is positively correlated with $\text{SHD}/|\edges|$ for CAM ($r=0.62, p<0.001$), but not for NOTEARS and LiNGAM. On both, synthetic and real data we find that a ranking of the causal discovery algorithms based on our metric agrees with a ranking based on $\text{SHD}/|\edges|$ (\tabrefs{tab:discovery_sachs,tab:discovery_synthetic}).
\begin{figure}[tb]
    \centering
    \begin{subfigure}{0.35\textwidth}
        \centering
        \includegraphics[width=\textwidth]{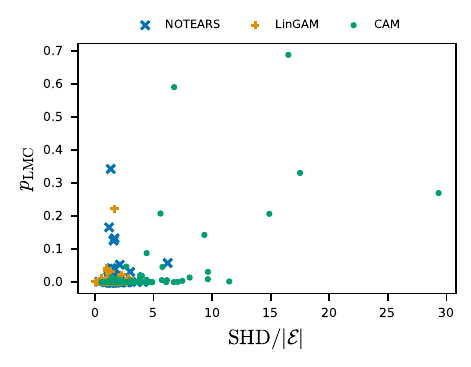}
        \vskip-0.5\baselineskip
        \caption{Synthetic (nonlinear mechanisms)}\label{fig:cd_p_lmc_synthetic_nonlinear}
    \end{subfigure}
    \hspace{2cm}
    \begin{subfigure}{0.35\textwidth}
        \centering
        \includegraphics[width=\textwidth]{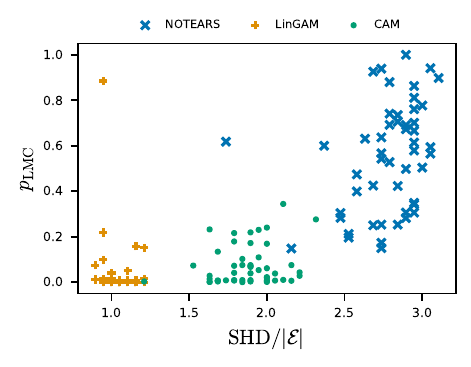}
        \vskip-0.5\baselineskip
        \caption{Protein signaling network data \cite{sachs2005}}\label{fig:cd_p_lmc_sachs}
    \end{subfigure}    
    \caption{\label{fig:results_causal_discovery}$p_\text{LMC}$ for $\Ggiven$ inferred via different causal discovery algorithms on (\subref{fig:cd_p_lmc_synthetic_nonlinear}) synthetic data and (\subref{fig:cd_p_lmc_sachs}) the protein signaling network data.}
    \vskip -\baselineskip
\end{figure}

\subsection{Comparison with Structural Intervention Distance}\label{sec:app_results_sid}
A common full-reference metric to evaluate the quality of some DAG is the Structural Intervention Distance (SID) \cite{peters2015}.

While our metric is not meant as an alternative to full-reference metrics such as SHD or SID, it is interesting to investigate whether our metric is correlated with the SID (similar to our experiments with SHD). In \figref{fig:p_lmc_vs_sid} we plot $p_\text{LMC}$ vs. $\text{SID}/|\edges|$ for graphs originating from a DE-$\verts$ for ER-10-2. Here we find that indeed $p_\text{LMC}$ is positively correlated with the relative SID ($r=0.67$) and graphs with higher domain knowledge have lower $p_\text{LMC}$ and relative SID on average, as expected.

\section{Proofs}\label{sec:app_proofs}
\subsection{Uniform sampling of permutations}\label{sec:proof_uniform_sampling}
In the following we proof \propref{prop:uniform_sampling}.
\begin{proof}
    From the orbit-stabilizer theorem~\citep{Rose2009} $|O(\Ggiven)|=|S_\Ggiven|/|\text{Stab}(\Ggiven)|$ and each DAG in $O(\Ggiven)$ has $|\text{Stab}(\Ggiven)|$ many different permutations as pre-images, namely the cosets of $\text{Stab}(\Ggiven)$ in $S_n$. 
    \end{proof}
\subsection{Type I error control of the node permutation test} \label{sec:proofpvalue}
In this section we proof \propref{prop:p_lmc_p-value} and show that 
$p_\text{LMC}$ is a p-value. 
\begin{proof}
Under the null hypothesis $\Ggiven$ is a randomly sampled graph from all node permutations. We need to show that then for any type I error control $\alpha$ we have that $\text{Pr}\left( p_\text{LMC}^{\Ggiven, \data}\leq \alpha\right) \leq \alpha.$ 
For ease of representation consider an increasing sorting of all permutations by their number of violations yielding $\sigma_1, \dots, \sigma_{|S|}$. For $\sigma_i$ at least $i/|S|$ permutations have at most as many violations as $\sigma_i$, hence $\text{Pr}_{\sigma}\left( |\vlmc[\sigma(\sigma_i(\Ggiven)), \data]| \leq |\vlmc[\sigma_i(\Ggiven), \data]|\right) \leq i/|S|$. 
We can view $\Ggiven$ as being sampled uniformly at random from $\sigma_i$. Thus, 
\begin{align}
    \text{Pr}\left(p_\text{LMC}^{\Ggiven, \data}\leq \alpha\right) &= \left|\left\{ \text{Pr}_{\sigma}\left( |\vlmc[\sigma(\sigma_i(\Ggiven)), \data]| \leq |\vlmc[\sigma_i(\Ggiven), \data]|\right) \leq \alpha \text{ for }1\leq i \leq |S|\right\}\right| / |S| \\ &\leq |\{ i/ |S| \leq \alpha \text{ for }1\leq i \leq |S|\}| / |S| \\ &\leq \alpha ~.
\end{align}
\end{proof}
In \figref{fig:emp_fpr_alpha} we also provide an empirical demonstration that \propref{prop:p_lmc_p-value} holds true in reality.
\begin{figure}[tb]
    \begin{minipage}{0.48\textwidth}
        \centering
        \includegraphics[width=0.7\textwidth]{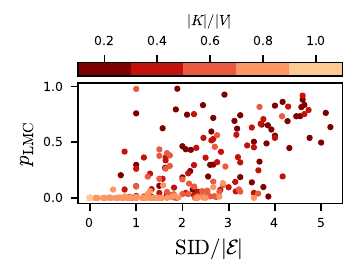}
        \caption{\label{fig:p_lmc_vs_sid} $p_\text{LMC}$ vs. SID for DAGs from a DE-$\verts$. Larger $|K|/|\verts|$ (brighter colors) correspond to more domain knowledge.}
    \end{minipage}\hfill
    \begin{minipage}{0.48\textwidth}
        \centering
        \includegraphics[width=0.6\textwidth]{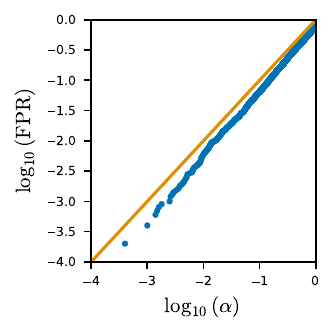}
        \caption{Empirical demonstration that \propref{prop:p_lmc_p-value} holds true in reality. We generated 10,0000 random ER-10-1 DAGs and randomly sampled one of the node permutations as $\hat{\mathcal{G}}$. We then plot FPR against $\alpha$ and find that $\text{FPR}\leq\alpha, \forall \alpha$.}
        \label{fig:emp_fpr_alpha}
    \end{minipage}
\end{figure}
\subsection{Markov equivalence class of permuted DAGs}\label{sec:proofmec}
In this section we proof~\propref{prop:markov_eq}.
\begin{proof}
    If $\vtpa[\sigma(\Ggiven), \Ggiven]=\emptyset$, then all parental triples in $\sigma(\Ggiven)$ are d-separated in $\Ggiven$. Further, \textit{all} d-separations in $\sigma(\Ggiven)$ also hold in $\Ggiven$ because every valid d-separation statement can be derived from the local ($X_i \CI_\graph \nd{i}{\graph} | \pa{i}{\graph}$) independencies (for a proof using semi-graphoid axioms see \eg \citet{lauritzen1990,lauritzen2020}), and the CIs implied by the parental triples imply the local independencies, \ie for some node $i$
    \begin{align}
        X_i \CI_\graph X_j | \pa{i}{\graph},~\forall X_j\in\nd{i}{\graph} \Rightarrow X_i \CI_\graph \nd{i}{\graph} | \pa{i}{\graph}
    \end{align}
    Since permutation preserves the number of valid d-separation statements, the statements in $\sigma(\Ggiven)$ cannot be a proper subset of the statements valid in $\Ggiven$ and thus all d-separation statements in $\Ggiven$ must also hold in $\sigma(\Ggiven)$. Thus, $\sigma(\Ggiven)$ and $\Ggiven$ are Markov equivalent.
\end{proof}

\subsection{Unconfounded cause-effect pairs}\label{sec:supp_proof_causal_effects}
In this section we proof~\thref{th:unconfounded_cause_effect}.
\begin{proof}
The proof is in strong analogy to the proof of the Theorem in \cite{atalanti2019}.
If $(X_i,X_j)$ is an unconfounded cause-effect pair, $X_i$ blocks all paths from $X_0$ to $X_j$. To show the converse, note that the dependence between $X_0$ and $X_i$ can only be due to a common cause or a directed path from $X_0$ to $X_i$ because  $X_0$ is not an effect of  $X_i$. If the dependence between $X_i$ and $X_j$ is due to a common cause or due to a directed path from $X_j$ to $X_i$, conditioning on $X_i$ unblocks the path from $X_0$ to $X_j$ and thus  $X_0 \CI X_j\,| X_i$ cannot hold.
\end{proof} 
\section{Implicit counting of faithfulness violations}\label{sec:supp_faithfulness}
This choice of the baseline also comes with another feature: so far, we have avoided counting violations of faithfulness because we hesitate to reject a hypothesis for that reason (see \secref{sec:validation_faithfulness} for a discussion of the faithfulness assumption). With our baseline, however, we implicitly also compare the number of faithfulness violations between a given DAG and its node-permutations. This is because every random DAG that shows more LMC violations than the given DAG must also show more violations of faithfulness (since it predicts the same number of independences as the given DAG). Hence, when using our baseline, we can remain undecided regarding our belief in faithfulness.
\section{Information-theoretic view on the node permutation test}\label{sec:supp_information_view}
Let us consider the following scenario. 
Let the variables of $\Ggiven$ be labeled according to some topological ordering of $\Ggiven$. This ordering may not be valid as topological ordering of 
$\graph^*$. Let $\pi$ some random choice among those re-orderings of nodes that result in a topological ordering of $\graph^*$.

We now state an assumption that we will later interpret with a grain of salt:

\begin{Assumption}\label{as:notworse}
$\sigma_\pi(\Ggiven)$
does not show more violations of LMC than $(\Ggiven)$.
\end{Assumption}

The idea is that $\Ggiven$ is not that far from $\graph^*$ that getting nodes into a valid causal ordering results in even more violations than the wrong causal ordering. Since we only estimate the number of violations for large $n$ anyway, we can read Assumption \ref{as:notworse} also in the sense of `does not show {\it significantly} more violations'.  

Together with Assumption \ref{as:notworse}, $\Ggiven$ and the independence structure reduces the set of possible orderings $\pi$ by the factor 
$p_\text{LMC}^{\Ggiven, \data}$. 
Starting with a uniform prior over $S_n$, 
$\Ggiven$ and the observed independences thus contain the Shannon information $- \log p_\text{LMC}^{\Ggiven, \data} $ about $\pi$. 
Likewise, $-\log p_\text{TPa}^{\Ggiven}$
is the amount of information $\Ggiven$
and the independence structure provide about the valid causal ordering.

\section{Sampling functional causal models to build a baseline}\label{sec:supp_fcm_baseline}
In this section we present an intuitive approach towards a baseline which is based on functional causal models. However, as we will see, this baseline is unsuitable without strong assumptions.

\subsection{Preliminaries}\label{sec:preliminaries}
The generative process of a system can be modelled using functional causal models (FCMs). We will denote with $\mathfrak{F}(\bm{\theta}, \graph)$ an FCM that defines conditional distributions of $\graph$ via a set of $n$ structural equations $X_i = f_i(\Pa{i}, N_i; \theta_i)$ with $f_i(\cdot;\theta_i), \theta_i \in \bm{\theta}$ denoting some parameterized function. We further denote with $P^{\mathfrak{F}(\bm{\theta},\graph)}$ the joint distribution induced by such an FCM.

\subsection{Validating faithfulness}\label{sec:validation_faithfulness}
According to Reichenbach's common cause principle, if two variables are (unconditionally) dependent, there must exist an unblocked path between them. 
The opposite implication, however, is only true if $\joint$ is faithful to $\graph$, \ie $A \CI_\joint B | C \Rightarrow A \CI_\graph B | C$ for all disjoint vertex sets $A, B, C$ \citep{spirtes2001}.

\begin{definition}[Marginal dependencies]\label{def:marginal_dependencies}
If $\joint$ is faithful to $\graph$ and there exists an unblocked path between variables $X_i, X_j$ then both must be (unconditionally) dependent, \ie $X_i \nCI_\joint X_j$.
\end{definition}

In \citet{reynolds2022} the authors validate a graph by measuring violations of marginal dependencies, \ie (pairwise) dependencies entailed by the graph which are not represented in the observed data. 
\begin{definition}[Violations of marginal dependencies]\label{def:violations_pd}
We denote with $\vmd[\Ggiven,\joint]$ the set of ordered pairs $(i\in\Ggiven, X_j\in\an{i}{\Ggiven})$ for which we observe marginal dependency (MD) violations on data $\data$, \ie
\begin{align}
    \vmd[\Ggiven,\data] = \left\{ i\in\Ggiven, X_j \in \an{i}{\Ggiven}: X_i \CI_\data X_j \right\} ~. \label{eq:V_md}
\end{align}
\end{definition}

$\vmd[\Ggiven, \data]$ detected via \defrefs{def:violations_pd} can have two causes: First, a wrong $\Ggiven$, where $X_i \nCI_\Ggiven X_j$ but $X_i \CI_\Gtrue X_j$. Second, a violation of faithfulness, where $X_i \CI_P X_j \not\Rightarrow X_i \CI_\Gtrue X_j$, or close-to-violation of faithfulness $X_i \CI_\data X_j \not\Rightarrow X_i \CI_\Gtrue X_j$. While detection of the first cause would be a viable measure of DAG consistency, it can't be distinguished from the second cause due to $\joint$ possibly being unfaithful (or close-to-unfaithful) to $\Gtrue$. The number of such faithfulness violations is a property of both the graph structure and the parameters of the system and thus highly domain dependent. Already in the population limit, the common `probability zero argument' relies on assuming a probability {\it density} in the parameter space of conditional distributions \citep{meek1995}. This argument breaks down for priors that assign nonzero probability to more structured mechanisms such as deterministic ones \citep{lemeire2012}. For finite data, this domain dependence gets even stronger because the probability of accepting independences not entailed by the Markov condition depends on the particular prior on the parameter space, and even for very simple priors over linear relations close-to-violations of faithfulness are not that unlikely \citep{uhler2013}.
The fact that those close-to-violations of faithfulness are both probable and undetectable when $\Gtrue$ is unknown renders $\vmd[\Ggiven,\data]$ (and with the same argument also possible measures of \textit{conditional} dependencies) unsuitable for measuring DAG consistency directly. 
\subsection{A baseline using functional causal models}
Previously discussed absolute metrics lack a baseline of how many violations to expect for the given graph. \Eg when we observe some number of faithfulness violations, it remains unclear, which of those are due to a wrong $\Ggiven$ and which are due to $P$ being (near) unfaithful to $\Gtrue$. One could consider constructing a seemingly straightforward baseline as follows: First, sample FCMs $\mathfrak{F}(\thgiven, \Ggiven)$ with random parameters $\thgiven$ and graph structure $\Ggiven$. 

Then, compare the number of violations for $\Ggiven$ and $\data$ to the number of violations observed for $\Ggiven$ and data \iid sampled from $P^{\mathfrak{F}(\thgiven, \Ggiven)}$ to decide whether the observed faithfulness violations are significant. However, the parameter prior of the sampled FCMs, as well as the chosen model class, plays a crucial role in the violations we observe.

Let us assume that $P$ is induced by some (generally unknown) FCM $\mathfrak{F}(\thtrue, \Gtrue)$. If we now sample an FCM $\mathfrak{F}(\thgiven,\Ggiven)$ with random parameters $\thgiven\sim\hat{\Theta}$, can we use 
\[
\mathbb{E}_{\thgiven\sim\hat{\Theta}}\left[\vlmc[\Ggiven,P^{\mathfrak{F}(\thgiven,\Ggiven)}]\right],~~
\mathbb{E}_{\thgiven\sim\hat{\Theta}}\left[\vmd[\Ggiven,P^{\mathfrak{F}(\thgiven,\Ggiven)}]\right],~~
\]
as a surrogate for the unknown $\vlmc[\Gtrue, P]$ or $\vmd[\Gtrue, P]$ respectively? The following example shows that for $\vmd[\Ggiven, P^{\mathfrak{F}(\thgiven,\Ggiven)}]$ this is not possible in general if the distributions of $\thgiven$ and $\thtrue$ have different priors.

\begin{example}\label{ex:generated_samples}
Assume some true DAG $\Gtrue$ with FCM $\mathfrak{T}=\mathfrak{F}(\thtrue, \Gtrue)$ being a linear-additive noise model, \ie functions $f_i$ defined by $\mathfrak{T}$ are of the form 
\begin{align}
    X_i = \sum_{j\in\Pa{i}} \thtrue_{ij} X_j + N_i, \label{eq:fcm_ex1}
\end{align}
with $\thtrue_{ij}\sim F^{*}$, with $F^{*}$ some probability measure with $\text{supp}(F^{*})=[-\epsilon, \epsilon], 0 \leq \epsilon \leq 1$, and $N_i\overset{\iid}{\sim}\mathcal{N}(\mu, \sigma^2)$. Further assume $\Ggiven$ = $\Gtrue$ and $\mathfrak{G} = \mathfrak{F}(\thgiven, \Ggiven)$ also being a linear-additive noise model, where $\thgiven_{ij}\sim \hat{F}$, with $\hat{F}$ another probability measure with $\text{supp}(\hat{F})=[-1, -\epsilon]\cup[\epsilon, 1]$. We would then find $\vmd[\Gtrue, P^\mathfrak{T}] \geq \vmd[\Ggiven, P^\mathfrak{G}]$.
\end{example}
\begin{proof}
Denote with $\hat{\vars}, \vars^{*}$ the multivariate normal variables introduced by the FCMs $\mathcal{G}, \mathcal{T}$, respectively. Since for all CI tests $X \CI Y | Z$ necessary to compute for $\vmd$ we have $Z=\emptyset$, \cf \eqref{eq:V_md}, we have $X \CI Y \Rightarrow \rho(X,Y)=0$. . 
Let us first look at the bivariate case $X\to Y$. Using \eqref{eq:fcm_ex1} and $N_i\overset{\iid}{\sim}\mathcal{N}(\mu, \sigma^2)$ we have $X^{*}=N_X$ and $Y^{*}=\theta^{*} X^{*} + N_Y$. Thus 
\begin{align}
    \text{cov}(X^{*}, Y^{*}) = \text{cov}(X^{*}, \theta^{*} X^{*} + N_Y) &= \theta^{*} \text{cov}(X,X) + \text{cov}(X^{*}, N_Y)\\\notag
    &= \theta^{*} \sigma^2%
\end{align}
\begin{align}
    \rho_{X^{*}, Y^{*}} = \frac{\text{cov}(X^{*}, Y^{*})}{\sqrt{\text{Var}(X^{*})\text{Var}(Y^{*})}} = \frac{\theta^{*}}{\sqrt{{\theta^{*}}^2 + 1}}
\end{align}
We can find $\rho_{\hat{X}, \hat{Y}}$ analogously and since $|\theta^{*}| < |\hat{\theta}|$ it follows $|\rho_{X^{*}, Y^{*}}| < |\rho_{\hat{X}, \hat{Y}}|$. Whether or not the smaller correlation leads to a violation of faithfulness depends on the parameter $\epsilon$ and the prespecified significance level. The above observation generalizes to arbitrary paths between nodes $i, j\in\an{i}{\graph}$ since by construction all $|\theta^{*}_{kl}| < |\hat{\theta}_{kl}|$.
\end{proof}

A random FCM will generally not entail distributions with similar statistical properties as the empirical distribution. This can be enforced by posing additional restrictions on the FCMs. \Eg to derive a surrogate baseline for the number of faithfulness violations we expect for the true DAG, one could fit FCMs $\mathfrak{G}$ to the observed data by choosing powerful models $f_i(\cdot; \theta_i)$, \eg realized by multi-layer perceptrons (MLPs) and learning the parameters $\theta_i$. If $\Ggiven$ = $\Gtrue$, we are guaranteed to observe $\vmd[\Ggiven, P^\mathfrak{G}] = \vmd[\Ggiven, \data]$ for sufficiently powerful function approximators $f_i$. However, whenever $\Ggiven \neq \Gtrue$, we are not guaranteed to do so. To see this, suppose $\Gtrue$ is given by $X \rightarrow Z, Y$ and $\Ggiven$ is given by $X \rightarrow Z \leftarrow Y$. The distribution $P^\mathfrak{G}$ of a fitted FCM $\mathfrak{G}$ will not satisfy Causal Minimality\footnote{A distribution P over a random vector $X=(X_1,\dots, X_n)$ which is Markovian over a graph $\graph$ satisfies Causal Minimality iff $\forall X_j, \forall Y \in \pa{j}{\graph}$  we have that $X_j \nCI_P Y \mid \pa{j}{\graph}\setminus\{Y\}$ \citep{peters2017}.} \wrt $\Ggiven$ and we will likely observe the faithfulness violation $Z \CI_{P^\mathfrak{G}} Y \not\Rightarrow Z\CI_\Ggiven Y$.
Therefore, we conclude that the aforedescribed approach of deriving a baseline using either sampled or fitted FCMs is not suitable in practice.
\section{Licenses for assets used in our work}\label{sec:app_licenses}
{\bfseries Datasets}
\begin{itemize}
    \item Auto MPG dataset: From the UCI Machine Learning Repository (\url{https://archive.ics.uci.edu/dataset/9/auto+mpg}) licensed under the Creative Commons Attribution 4.0 International (CC BY 4.0) license.
    \item Sachs dataset: From the supplementary material of \citet{sachs2005} \url{https://www.science.org/doi/10.1126/science.1105809?url_ver=Z39.88-2003&rfr_id=ori:rid:crossref.org&rfr_dat=cr_pub%20%200pubmed}. We could not find any licensing information.
    \item APM dataset: The dataset will be published under the Apache-2.0 license.
\end{itemize}

{\bfseries Causal Discovery Algorithms}
\begin{itemize}
    \item LiNGAM: From the \textit{Causal Discovery Toolbox} (\url{https://github.com/FenTechSolutions/CausalDiscoveryToolbox}) licensed under the MIT license.
    \item CAM: From the \textit{Causal Discovery Toolbox} (\url{https://github.com/FenTechSolutions/CausalDiscoveryToolbox}) licensed under the MIT license.
    \item NOTEARS: From the author's official repository (\url{https://github.com/xunzheng/notears}) licensed under the Apache-2.0 license.
\end{itemize}

{\bfseries Conditional independence tests}
\begin{itemize}
    \item KCI: From \textit{DoWhy} (\url{https://github.com/py-why/dowhy}) licensed under the MIT license.
    \item GCM: From \textit{DoWhy} (\url{https://github.com/py-why/dowhy}) licensed under the MIT license.
\end{itemize}

\end{document}